\documentclass[journal]{IEEEtran} %draftcls 
\usepackage{cite}
\usepackage[pdftex]{graphicx}
\usepackage{wrapfig}
\usepackage{epstopdf}
\usepackage[caption=false,font=footnotesize]{subfig}
\usepackage[cmex10]{amsmath}
\usepackage{amssymb}
\usepackage{amsthm}
\usepackage{algorithm,algorithmic}
\usepackage{xspace}
\usepackage{color}
\usepackage{comment} 
\usepackage{soul}

\def\ben{\begin{equation}}
\def\een{\end{equation}}
\def\bmx{\begin{bmatrix}}
\def\emx{\end{bmatrix}}
\def\N{\mathcal{N}}
\def\E{\mathbb{E}}

\def\R{\mathbb{R}}
\def\I{\mathbf{I}}
\def\C{\mathcal{C}}
\def\Xs{{\mathcal X}}
\def\Ys{{\mathcal Y}}
\def\X{{\mathbf X}}

\def\Er{\mathcal{E}}
\def\0{\mathbf{0}}
\def\1{\mathbf{1}}
\def\dist{\mathcal D}
\def\x{\mathbf{x}}
\def\y{\mathbf{y}}
\def\z{\mathbf{z}}
\def\n{\mathbf{n}}

\def\u{\mathbf{u}}

\def\T{\mathcal{T}}
\def\U{\mathbf U}
\def\V{\mathbf V}
\def\A{{\bf A}}
\newcommand{\chcolor}{black}

\def\tr{\operatorname{tr}}

\def\pdet{\operatorname{pdet}}
\def\diag{\operatorname{diag}}
\def\vect{\operatorname{vec}}
\def\Pr{\mathbb{P}\rm{r}}
\DeclareMathOperator*{\argmin}{arg\,min}
\newtheorem{remark}{Remark}

\newtheorem{theorem}{Theorem}

\newtheorem{proposition}{Proposition}
\newtheorem{lemma}{Lemma}

\def\trait{TRAIT\xspace}
\def\lrt{LRT\xspace}

\begin{document}
\title{The Role of Principal Angles in \\Subspace Classification}

\author{Jiaji Huang,~\IEEEmembership{Student Member,~IEEE,}
	Qiang Qiu,
	Robert~Calderbank,~\IEEEmembership{Fellow,~IEEE}
\thanks{The authors are with the Department
of Electrical and Computer Engineering, Duke University, Durham,
NC, 27708 USA e-mail:jiaji.huang@duke.edu, qiang.qiu@duke.edu, robert.calderbank@duke.edu}% <-this % stops a space
%\thanks{Manuscript received , 2015; revised , 2015.}
}

% The paper headers
%\markboth{IEEE transaction on Signal Processing,~Vol.~, No.~, December~2014}%
%{Huang \MakeLowercase{\textit{et al.}}: Classification of Subspaces}

\maketitle
\begin{abstract}
%	We analyze the performance of Maximum a posteriori and Nearest Subspace classifiers that are aware of subspace geometry. 
%	Gaussian Mixture Model is first adopted for an analysis of the optimal Maximum a Posteriori classifier. 
%	A more general case relaxing the Gaussian assumption is then considered, 
%	in which the nearest subspace classifier is characterized. 
%	In both cases, explicit relationships are established between the classification error and 
%	the principal angles of the different subspaces.
%	Motivated by the theoretical results, a learned transform that increases the inter-class separation is designed.
%	Competitive experimental results are achieved on both synthetic and measured datasets for signal classification tasks.

	Subspace models play an important role in a wide range of signal processing tasks, and this paper explores how the pairwise geometry of subspaces influences the probability of misclassification. When the mismatch between the signal and the model is vanishingly small, the probability of misclassification is determined by the product of the sines of the principal angles between subspaces. When the mismatch is more significant, the probability of misclassification is determined by the sum of the squares of the sines of the principal angles. Reliability of classification is derived in terms of the distribution of signal energy across principal vectors. Larger principal angles lead to smaller classification error, motivating a linear transform that optimizes principal angles. The transform presented here (TRAIT) preserves some specific characteristic of each individual class, and this approach is shown to be complementary to a previously developed transform (LRT) that enlarges inter-class distance while suppressing intra-class dispersion. 
	{\color{\chcolor}
	Theoretical results are supported by demonstration of superior classification accuracy on synthetic and measured data even in the presence of significant model mismatch.
	}
\end{abstract}


\IEEEpeerreviewmaketitle

%========================================================================================
\section{Introduction}
%========================================================================================
\IEEEPARstart{S}{ignals}
that are nominally high dimensional often exhibit a low dimensional geometric structure. 
For example, fixed-pose images of human faces are recorded using more than 1000 pixels, but can be represented by a 9-dimensional harmonic subspace~\cite{Basri2003}. 
Motion trajectories of a rigid body might be recorded by hundreds of sensors, but must intrinsically be represented by a 4-dimensional subspace~\cite{Rao2008}. 
There are many more examples where a low-dimensional subspace model captures intrinsic geometric structure, ranging from 
user ratings in a recommendation system~\cite{Rennie2005} to signals emitted by multiple sources impinging at an antenna array~\cite{Scharf2002}.
The subspace geometry has assisted tasks of interest to both signal processing~\cite{Xie2013,Petrels2013} and machine learning communities\cite{Vidal2009,QiangJMLR}.

A Gaussian Mixture Model (GMM) measures proximity to a union of linear or affine subspaces, by imposing a low-rank structure on the covariance of each mixture component. It can be used to approximate a nonlinear manifold by fitting mixture components to local patches of the manifold~\cite{minhua2010,Xie2013}, hence providing a high fidelity representation of a wide variety of signal geometries. The simplicity of the model facilitates signal reconstruction~\cite{BCS2008,Yu2011,Yu2012,FracescoRecon}, making GMMs a very attractive signal source model in compressed sensing. 
The value of low-rank GMMs extends to classification, where each class is modeled as a low-rank mixture component, and classes are identified by their projections onto linear features. Optimal feature design is addressed in ~\cite{MinhuaLDA,QiangJMLR}. 

The GMM is usually only an approximation to the truth. For example, the full spectrum associated with a face image follows a power law distribution, and when we truncate to the first 9 harmonic dimensions, the residual energy will be a source of error in classification. Even if the true model were a GMM, we can only learn an approximation to the true model from training data. The more data we see, the better is the fit of our empirical model, but some degree of mismatch is unavoidable. If we treat this mismatch as a form of noise, then we can use information theory to derive fundamental limits on the number of classes that can be discerned (see~\cite{MattCapacity} for more details).

%This paper is concerned with signal classification tasks aware of the subspace structure. 
%Although substantial success has been achieved in both application and theory,
%we noticed that there still lacks a crystal clear relationship between the classifier's performance and the subspace geometry.
%Moreover, existing analysis~\cite{FracescoRate,MattCapacity}, often relies much on the Chernoff bound\cite{Duda2000}, which is only suitable for analyzing the performance of a Maximum a Posteriori (MAP) classifier.
%However, MAP requires the signal's distribution to be perfectly known. This makes the MAP less practically useful because there are many cases where training data are not abundant for model (e.g., the GMM) learning.
%These facts constitute the motivation of our work and the problems we want to address in this paper.
%The contributions of this paper are three-fold:

This paper explores how the pairwise geometry of subspaces influences the probability of misclassification. There are parallels with non-coherent wireless communication~\cite{Hochwald2000}, where information is encoded as a subspace drawn from a fixed alphabet, and the function of the receiver is to distinguish the transmitted subspace.
When each component is perfectly modeled as a Gaussian, the performance of the MAP classifier can be analyzed using the Chernoff Bound~\cite{Duda2000}. When fidelity is perfect, there is no mismatch, and fundamental limits on performance are determined by the rank of the intersection of the classes~\cite{FracescoRate,MattCapacity}.

In this paper, we further consider how best to discriminate classes, when the alignment between the GMM model and the data is only approximate. We make three main contributions in this paper:

%The GMM assumption facilitates analysis using Chernoff bound, which tightly bounds the error probability of a MAP classifier. However, there are many more cases where a MAP classifier is not so practically useful. In these cases, a Nearest Subspace Classifier may be a good alternative. Motived by this fact, we characterize the performance of the Nearest Subspace Classifier as well.

\begin{enumerate}
	\item 
	We express the probability of pairwise misclassification in terms of the principal angles between the corresponding subspaces.
	\textcolor{\chcolor}{
	This expression depends on the mismatch between the signal and the model. Interpreting this mismatch as noise, we provided analysis of the low, moderate, and high SNR regimes.}
	This improves upon \cite{FracescoRate}, in the sense that we have a more explicit expression of the ``measurement gain" as proposed in \cite{FracescoRate}.
	\item 
We characterize the probability of misclassification for more general distributions near subspaces.
This is motivated by the case where training samples per class are insufficient for a reliable estimate of covariance. 
In these cases, we have very little knowledge about the signal's distribution and a MAP classifier is not good fit.
%Nearest Subspace Classifier (NSC) is an alternative. We bound the misclassification probability of the NSC rather than the MAP classifier.
The Nearest Subspace Classifier (NSC) provides an alternative and we use the NSC classifier rather than the MAP to bound the probability of misclassification.
	\item 
We develop a feature extraction method, TRAIT, that 
\textcolor{\chcolor}{effectively enlarges principal angles between different subspaces and preserves intra-class structure. We demonstrate superior classification accuracy on synthetic and measured data, particularly in the presence of significant model mismatch.}
%takes advantage of subspace structure within a class and also between classes, and we demonstrate state-of-the-art performance on synthetic and measured datasets.
\end{enumerate}

This paper is organized as follows. Section~\ref{sec:pre} presents the subspace geometry framework.
Section~\ref{sec:MAP} analyzes the Maximum a Posteriori (MAP) classifier under the GMM assumption.
Section~\ref{sec:NSC} analyzes the performance of Nearest Subspace Classifer (NSC), which relaxes the GMM assumption.
Section~\ref{sec:trait} proposes a feature extraction method, TRAIT, that exploits subspace geometry, and presents experimental results for both synthetic and \textcolor{\chcolor}{measured} datasets. Section~\ref{sec:conclu} provides a final summary.

A note on notations: we use bold upper case letters for matrices, e.g., $\bf X$, and bold lower case letters for vectors, e.g., $\bf x$. The transpose of a matrix $\bf X$ is denoted by $\bf X^\top$. 
Scalars are written as plain letters, e.g., $\lambda$, $K$.

%========================================================================================
\section{Geometric Framework}
\label{sec:pre}
%========================================================================================
Consider two subspaces $\Xs$ and $\Ys$ of $\R^n$ with dimensions $\ell$ and $s$ respectively, where $\ell\leq s$.
The principal angles between $\Xs$ and $\Ys$, denoted as $\theta_1,\dots,\theta_\ell$,
are defined recursively as follows
\[\begin{array}{ll}\theta_1=\min_{\x_1\in\Xs, \y_1\in\Ys}\arccos \left(\x_1^\top \y_1\over\|\x_1\|\|\y_1\| \right),&
\\
\qquad\vdots&\\
\theta_j=\min_{\substack{\x_j\in\Xs, \y_j\in\Ys\\ \x_j\perp \x_1,\dots,\x_{j-1}\\ \y_j\perp \y_1,\dots,\y_{j-1}}}
\arccos\left(\x_j^\top \y_j \over \|\x_j\|\|\y_j\|\right),& j=2,\dots,\ell.\\
\end{array}\]
The vectors $\x_1,\dots,\x_\ell$ and $\y_1,\dots,\y_\ell$, are called principal vectors. 
The dimension of $\Xs\cap\Ys$ is the multiplicity of zero as a principal angle.
It is straightforward to compute the principal angles by calculating the singular values of $\bf X^\top \bf Y$, where $\bf X$ and $\bf Y$ are orthonormal bases for $\Xs$ and $\Ys$ respectively. 
The singular values of $\bf X^\top \bf Y$ are then $\cos\theta_1,\dots,\cos\theta_\ell$. 

Let $\ell=s$. The principal angles induce several distance metrics \textcolor{\chcolor}{on the Grassmann manifold},
of which the most widely used is the (squared) chordal distance $\dist_c^2(\Xs,\Ys)$~{\color{\chcolor}\cite{Edelman1998}}, given by 
\[\dist_c^2(\Xs,\Ys)=\sum_{i=1}^s \sin^2\theta_i.\]
The chordal distance is an aggregate, and in the following sections we will see how probability of misclassification depends, not so much on this aggregate, but on the individual principal angles.

%========================================================================================
\section{The MAP Classifier for a GMM}
\label{sec:MAP}
%========================================================================================
We begin by considering the MAP classifier, which is optimal when the signal distribution is known. We focus on binary classification, where the two classes are equiprobable, since the generalization from two to many classes is well understood \cite{Chenhao2012,FracescoRate}.

We model each class as zero mean Gaussian distributed, where the covariance is near low-rank. 
%The setup
Classification
can be formulated as the following binary hypothesis testing problem
\ben\begin{array}{lcl}
	H_1: \x&\sim&\N(\0,\boldsymbol \Sigma_1)\\
	H_2: \x&\sim&\N(\0,\boldsymbol \Sigma_2).
	\label{eq:hypo}
\end{array}
\een
We justify the zero-mean assumption by observing that in applications such as face recognition~\cite{MaYiSRC}, or motion trajectory segmentation~\cite{Rao2008}, the actual mean is considered as a nuisance parameter, and is removed prior to processing. Given the near-subspace assumption, we model the two covariances as
\ben
\begin{aligned}
	\boldsymbol \Sigma_1 &= \U_1\boldsymbol\Lambda_1\U_1^\top+\sigma^2\I \\
	\boldsymbol \Sigma_2 &= \U_2\boldsymbol\Lambda_2\U_2^\top+\sigma^2\I.
\end{aligned}
\label{eq:noiseModel}
\een
where $\U_1, \U_2\in\R^{n\times d}$ are the orthonormal bases for the two signal subspaces, 
denoted by $\Xs_1$ and $\Xs_2$.  Typically $n\gg d$.
$\boldsymbol\Lambda_1, \boldsymbol\Lambda_2\in\R^{d\times d}$ are \textcolor{\chcolor}{diagonal matrices of} eigenvalues. 
We assume that the two subspaces have the same dimension $d$, and that the diagonal elements of $\boldsymbol\Lambda_1$, $\boldsymbol\Lambda_2$ are arranged in descending order. In the application to motion trajectories we take $d=4$, and in the application to face recognition we might take $d=9$. Denote the $i$-th largest eigenvalue of $\boldsymbol\Lambda_j$ by $\lambda_{j,i}$. Finally let $\sigma^2$ be the variance of the noise, which quantifies the degree of mismatch between the subspace model and the data.

\textcolor{\chcolor}{
Denote the probability of mistaking hypothesis 2 for hypothesis 1 by ${\Pr}(H_2|H_1)$, and define ${\Pr}(H_1|H_2)$ similarly.
%Denote ${\Pr}(H_2|H_1)$ as the probability of mistaken hypothesis 2 for 1 and define ${\Pr}(H_1|H_2)$ similarly.
Under the assumption that the two hypotheses are equiprobable, the error probability $P_e$ of a MAP (optimal) classifier is
\ben
\begin{aligned}
	P_e &= {1\over2}[{\Pr}(H_2|H_1)+{\Pr}(H_1|H_2)]\\
	&={1\over2}\int \min({\Pr}(\x|H_1),{\Pr}(\x|H_2))d\x
\end{aligned}
\een
Since this integral does not admit a closed form solution, we study the Bhattacharyya upper bound~\cite{Bbound} to $P_e$ instead. This bound is a special case of the Chernoff bound~\cite{Duda2000} derived using the observation $\min(a,b)\leq\sqrt{ab}$.} The Bhattacharyya bound gives
%Therefore we study the Bhattacharyya upper bound~\cite{Bbound} (a special case of the Chernoff bound~\cite{Duda2000}
%by leveraging the fact $\min(a,b)\leq\sqrt{ab}$) to $P_e$ instead, which is
\ben
P_e\leq {1\over2}e^{-K},~\mbox{where } K={1\over 2}\ln\frac{\det\left(\frac{\boldsymbol\Sigma_1+\boldsymbol\Sigma_2}{2}\right)}{\sqrt{\det\boldsymbol\Sigma_1\cdot\det\boldsymbol\Sigma_2}}.
\label{eq:Bbound}
\een
The numerator inside the logarithm measures the volume of space occupied by both subspaces together, and the denominator measures the volumes occupied separately. These quantities depend on the principal angles, and we now study the performance of the Bhattacharyya bound in the high, low \textcolor{\chcolor}{and moderate} SNR regimes.

%~~~~~~~~~~~~~~~~~~~~~~~~~~~~~~~~~~~~~~~~~
\subsection{The High SNR Regime}
\label{sec:highSNR}
%~~~~~~~~~~~~~~~~~~~~~~~~~~~~~~~~~~~~~~~~~

We first consider the case when $\sigma^2\to0$, which means that the mismatch between the signal and the model becomes vanishingly small. 
Since the intersection $\Xs_1\cap\Xs_2$ between the two subspaces plays a special role, we write the two covariances as
\ben
\begin{array}{l}
	\boldsymbol\Sigma_1 = \U_{1,\cap}\boldsymbol\Lambda_{1,\cap} \U_{1,\cap}^\top+\U_{1,\backslash}\boldsymbol\Lambda_{1,\backslash} \U_{1,\backslash}^\top+\sigma^2\I,\\
	\boldsymbol\Sigma_2 = \U_{2,\cap}\boldsymbol\Lambda_{2,\cap} \U_{2,\cap}^\top+\U_{2,\backslash}\boldsymbol\Lambda_{2,\backslash} \U_{2,\backslash}^\top+\sigma^2\I
\end{array}
\label{eq:Sigs}
\een 
Here both $\U_{1,\cap}\in\R^{n\times r}$ and $\U_{2,\cap}\in\R^{n\times r}$ span $\Xs_1\cap\Xs_2$ with singular values 
$\boldsymbol\Lambda_{1,\cap}$ and $\boldsymbol\Lambda_{2,\cap}$ respectively. 
$\U_{1,\backslash}\in\R^{n\times (d-r)}$ spans $\Xs_1\backslash\Xs_2$ with singular values $\boldsymbol\Lambda_{1,\backslash}$.
And $\U_{2,\backslash}\in\R^{n\times (d-r)}$ spans $\Xs_2\backslash\Xs_1$ with singular values $\boldsymbol\Lambda_{2,\backslash}$.

The following theorem bounds the classification error in the high SNR regime.
\begin{theorem}
	\label{thm:highSNR}
	Assume $n\geq 2(d-r)$.
	As $\sigma^2\rightarrow 0$, the classification error is upper bounded as
	\[
	P_e\leq 
	c_1 (\sigma^2)^{d-r\over2}\left(\prod_{i=r+1}^d\sin^2\theta_i\right)^{-{1\over2}}+o\left((\sigma^2)^{d-r\over2}\right)
	\]
	where ``$g(\sigma^2)=o(f(\sigma^2))$" stands for $\lim_{\sigma^2\to0}\frac{g(\sigma^2)}{f(\sigma^2)}=0$. 
	The constant $c_1$ is given by,
	\[\begin{aligned}
	c_1=2^{{2d-r\over2}-1}&\left[ \frac{\pdet(\U_{1,\cap}\boldsymbol\Lambda_{1,\cap}\U_{1,\cap}^\top+\U_{2,\cap}\boldsymbol\Lambda_{2,\cap}\U_{2,\cap}^\top)}{\sqrt{\prod_{i=1}^r\lambda_{1,\cap,i}\cdot\prod_{i=1}^r\lambda_{2,\cap,i}}} \right.\\
	&\quad\left.\cdot
	\prod_{i=1}^{d-r}\sqrt{\lambda_{1,\backslash,i}\cdot\lambda_{2,\backslash,i}} \right]^{-{1\over2}}
	\end{aligned}\]
	where $\pdet$ denotes the pseudo-determinant.
\end{theorem}
\begin{proof}
	The method is to expand the Bhattacharyya bound in terms of principal angles, and the details are provided in Appendix \ref{proof:highSNR}. 
\end{proof}
\begin{remark}
	\begin{enumerate}
		\item
			Typically $n\gg d$ for measured data, so the condition $n\geq 2(d-r)$ is usually satisfied.
		\item
			The classification error is upper bounded by $(\sigma^2)^{d-r\over2}$; the smaller the overlap between subspaces, the easier it is to discriminate between classes. When two subspaces overlap completely, there is an error floor.
	\end{enumerate}
\end{remark}
There is a duality between the GMM classification problem and multiple antenna communication~\cite{Tarokh1998}.
In multiple antenna communications, a codeword is a $d\times n$ array, where the rows are indexed by transmit antennas, the columns are indexed by time slots in a data frame, and the entries are the symbols to be transmitted. The probability of mistaking codeword $C_i$ for codeword $C_j$, ${\Pr}(i\to j)$, satisfies
\[
	{\Pr}(i\to j) \leq (\sigma^2/2)^k(1/\lambda_1^2\dots\lambda_k^2),
\]
where $k$ is the rank of $C_i-C_j$, whose singular values are $\lambda_1,\dots,\lambda_k$. 
The primary objective in code design for multiple antenna wireless communication is to maximize the minimum rank of the difference between distinct codewords. If the minimum rank is $k$, the code is said to achieve a diversity gain of $k$. 

An important secondary objective in code design for multiple antenna wireless communication is to maximize the minimum product of the singular values of the difference between distinct codewords. This minimum product determines the coding gain.
			
The counterpart of coding gain in classification is the product of sines of the principal angles. This quantity determines the intercept of the error exponent with the vertical axis. The smaller the energy in the intersection of the subspaces, the smaller is the classification error. The larger the principal angles, the smaller is the classification error.

%~~~~~~~~~~~~~~~~~~~~~~~~~~~~~~~~~~~~~~~~~
\subsection{The Low SNR Regime}
%~~~~~~~~~~~~~~~~~~~~~~~~~~~~~~~~~~~~~~~~~
\label{sec:lowSNR}
This is the case where the noise variance $\sigma^2$ and the singular values are commensurable; in other words, the mismatch between the signal and the empirical model cannot be neglected.
The MAP classifier in this case is characterized by the following theorem.
\begin{theorem}
	\label{thm:lowSNR}
	When $\sigma^2$ is sufficiently large, the Bhattacharyya upper bound is sandwiched between 
	\[\underline{P_e}^{UB}={1\over2}\exp\left\{-{1\over\sigma^4}\left(c_2-{1\over 16}\lambda_{1,1}\lambda_{2,1}\sum_{i=1}^d\cos^2\theta_i \right)\right\}\]\\
	and
	\[\overline{P_e}^{UB}={1\over2}\exp\left\{-{1\over\sigma^4}\left(c_3-{1\over 8}\lambda_{1,1}\lambda_{2,1}\sum_{i=1}^d\cos^2\theta_i \right)\right\},\]
	where $\overline{P_e}^{UB}>\underline{P_e}^{UB}$. And the constants $c_2$ and $c_3$ are given by
	\[
	\begin{aligned} 
		c_2 =& {\sigma^4\over4}\left[\sum_{i=1}^d{\lambda_{1,i}\over\sigma^2}-{1\over2}\sum_{i=1}^d \left({\lambda_{1,i}\over 2\sigma^2}\right)^2-\sum_{i=1}^d\ln\left(1+{\lambda_{1,i}\over\sigma^2}\right) \right]\\
		&+{\sigma^4\over4}\left[\sum_{i=1}^d{\lambda_{2,i}\over\sigma^2}-{1\over2}\sum_{i=1}^d \left({\lambda_{2,i}\over 2\sigma^2}\right)^2
		-\sum_{i=1}^d\ln\left(1+{\lambda_{2,i}\over\sigma^2}\right) \right]
	\end{aligned}\]
	\[\begin{aligned}
		c_3=& {\sigma^4\over4}\left[\sum_{i=1}^d{\lambda_{1,i}\over\sigma^2}-\sum_{i=1}^d \left({\lambda_{1,i}\over 2\sigma^2}\right)^2
		-\sum_{i=1}^d\ln\left(1+{\lambda_{1,i}\over\sigma^2}\right) \right]\\
		&+{\sigma^4\over4}\left[\sum_{i=1}^d{\lambda_{2,i}\over\sigma^2}-\sum_{i=1}^d \left({\lambda_{2,i}\over 2\sigma^2}\right)^2
		-\sum_{i=1}^d\ln\left(1+{\lambda_{2,i}\over\sigma^2}\right) \right].
	\end{aligned}
	\]
\end{theorem}
\begin{proof}
	The details are given in appendix~\ref{proof:lowSNR}.
\end{proof}

\begin{remark}
	The dimension of the overlap between the two subspaces plays a less important role in the low SNR regime, and classification error is a function of chordal distance. This gives rise to an interesting duality between GMM model based classification and the space-time decoding~\cite{Packing_RMcodes}, where error probability is influenced by product or sum diversity in high or low SNR regime respectively.
\end{remark}

%~~~~~~~~~~~~~~~~~~~~~~~~~~~~~~~~~~~~~~~~~
\subsection{\textcolor{\chcolor}{The Moderate SNR Regime}}
\label{sec:moderateSNR}
%~~~~~~~~~~~~~~~~~~~~~~~~~~~~~~~~~~~~~~~~~
{\color{\chcolor}
We now consider a moderate noise/mismatch regime, where ${p\over c(p)}\leq\frac{\lambda_{1,j}}{\sigma^2}, \frac{\lambda_{2,j}}{\sigma^2}\leq p$ for $j=1,\dots,d$ and $p>1,c(p)>1$.
Moderate SNR also implies that $p$ is not very large.

The most important element in the analysis of classification error is to lower bound the term $\ln\det\left(\frac{\boldsymbol{\Sigma}_1+\boldsymbol{\Sigma}_2}{2}\right)$ in Eq.~\eqref{eq:Bbound},
\[
\begin{aligned}
	\ln\det\left(\frac{\boldsymbol{\Sigma}_1+\boldsymbol{\Sigma}_2}{2}\right) =& \ln\det\left(\I+\frac{\U_1\boldsymbol{\Lambda}_1 \U_1^\top+\U_2\boldsymbol{\Lambda}_2 \U_2^\top}{2\sigma^2}\right)\\
	&+n\ln\sigma^2.
\end{aligned}
\]
Denote the non-zero singular values of $\mathbf{D}\triangleq{1\over2\sigma^2}(\U_1\boldsymbol{\Lambda}_1 \U_1^\top+\U_2\boldsymbol{\Lambda}_2 \U_2^\top)$ by $\lambda_1,\dots,\lambda_{2d-r}$. Then 
\ben
\ln\det\left(\frac{\boldsymbol{\Sigma}_1+\boldsymbol{\Sigma}_2}{2}\right)=\sum_{i=1}^{2d-r}\ln(1+\lambda_i)+n\ln(\sigma^2).
\een
The following lemma provides a lower bound on $\ln(1+\lambda_i)$.
\begin{lemma}
	\label{lemma:ell(p)}
	There exists $0\leq L<{p-1\over2}$ such that for any $\lambda_i\in[L,p]$, 
	\ben
	\ln(1+\lambda_i)\geq \ln(1+p)+{1\over 1+p}(\lambda_i-p)-{1\over(1+p)^2}(\lambda_i-p)^2.
	\label{eq:expansion_p}
	\een 
\end{lemma}
\begin{proof}
	See Appendix~\ref{sec:proofModerateSNR}.
\end{proof}
Let $L(p)$ be the smallest possible value of $L$, define $c(p)={p\over2L(p)}$ if $L(p)>0$ and $c(p)=+\infty$ if $L(p)=0$. Note that $c(p)>1$ since $L(p)<{p-1\over2}$.
\begin{theorem}
	\label{thm:moderateSNR}
	If ${p\over c(p)}\leq {\lambda_{1,i}\over\sigma^2}, {\lambda_{2,i}\over\sigma^2}\leq p$,
	then the classification error is upper bounded as
	\[
	P_e\leq {1\over2}\exp\left\{-c_4(2d-r)+{\lambda_{1,1}\lambda_{2,1}\over4\sigma^4(1+p)^2}\sum_i\cos^2\theta_i+c_5  \right\},
	\]
	where $c_4={1\over2}\left[\ln(1+p)-{p\over1+p}-{p^2\over(1+p)^2} \right]$ and $c_5$ depends on $p$ and $\lambda_{1,i}\over\sigma^2$, $\lambda_{2,i}\over\sigma^2$.
	%	$c_5=-{1+3p\over2(1+p)^2}\sum_i \left({\lambda_{1,i}\over\sigma^2}+{\lambda_{2,i}\over\sigma^2}\right)+{1\over 4(1+p)^2}\sum_i\left({\lambda_{1,i}\over \sigma^2} \right)^2+\left({\lambda_{2,i}\over \sigma^2} \right)^2$
\end{theorem}
\begin{proof}
	See Appendix~\ref{sec:proofModerateSNR}.
\end{proof}
\begin{comment} 
The value $p\over2L(p)$ is strictly bigger than $1$ by noticing that $L(p)\leq {p-1\over2}$ in Lemma~\ref{lemma:ell(p)}. In fact, a simulation study can characterize $p\over2L(p)$ as Fig.~\ref{fig:p_ellp}.
\begin{wrapfigure}{i}{0.3\columnwidth}
	\vspace{-20pt}
	\begin{center}
		\includegraphics[width=0.3\columnwidth]{figs/p_ellp.eps}
	\end{center}
	\vspace{-15pt}
	\caption{$p\over2L(p)$ for different $p$ values. It is always strictly bigger than 1.}
	\vspace{-10pt}
	\label{fig:p_ellp}
\end{wrapfigure}
As two concrete examples, when $p=4$ or $5$, ${p\over2L(p)}=3.44$ or $2.79$. 
Correspondingly, as long as $\lambda_{1,i}/\sigma^2,\lambda_{2,i}/\sigma^2\in[1.16,4]$ or $[1.79,5]$,
theorem~\ref{thm:moderateSNR} holds.
Both cases are reasonable to be considered as in the moderate SNR regime. 
The form of the upper bound suggests that in this moderate SNR regime, the role of chordal distance is more important than the product of the sines of the principal angles.
\end{comment} 
\begin{remark}
	It is straightforward to show numerically that $c(p)=3.44$, $2.79$ for $p=4,5$ respectively, that $c(p)\geq 2.02$ for $p\leq 10$, and that $c(p)\geq 1.61$ for $p\leq 100$.
	The form of the upper bound suggests that in the moderate SNR regime, the role of chordal distance is more important than the product of the sines of the principal angles.
\end{remark}
}

\subsection{Numerical Analysis of Synthetic Data}
\label{sec:MAP_nume}
%~~~~~~~~~~~~~~~~~~~~~~~~~~~~~~~~~~~~~~~~~
We explore the difference between classification in the low and high SNR regimes through a simple numerical example. Consider the following pairs of subspaces: \\
case 1:
\[
\U_1=\bmx 1 & 0 & 0 & 0 \\0 &1 &0 &0 \emx^\top \quad \U_2=\bmx 1 & 0 & 0 & 0 \\0 &0 &1 &0 \emx^\top.
\]
case 2:
\[
\U_1=\bmx 1 & 0 & 0 & 0 \\0 &1 &0 &0 \emx^\top \quad \U_2={1\over\sqrt{2}}\bmx 1 & 0 & 0 & -1 \\0 &1 &1 &0 \emx^\top.
\]
We set $\boldsymbol\Lambda_1=\boldsymbol\Lambda_2=\I$ for both cases. 
In case 1, the two principal angles are $\theta_1=0, \theta_2=\pi/2$ and in case 2, the two principal angles are $\theta_1=\pi/4, \theta_2=\pi/4$. The chordal distances in these two cases are the same, but in case 1 the product of sines of non-zero principal angles is 1, whereas in case 2 it is 1/2.
However, there is a nontrivial intersection dimension in case 1. 
The product of nonzero sine principal angles is $1$ for case 1, and ${1\over2}$ for case 2.

We vary the degree of mismatch $\sigma^2$, and evaluate the bounds developed in the above \textcolor{\chcolor}{three} theorems.
In the high SNR regime, we plot the empirical misclassification probability $P_e$ with the value $
c_1 (\sigma^2)^{d-r\over2}\left(\prod_{i=r+1}^d\sin^2\theta_i\right)^{-{1\over2}}$ given in Theorem \ref{thm:highSNR}.
In the low SNR regime, we plot the upper bound $\overline{P_e}^{UB}$ in Theorem \ref{thm:lowSNR}.
\textcolor{\chcolor}{
In the moderate SNR regime, we take $p=6$, and we vary
$\sigma^2$ between ${1\over p}$ and ${c(p)\over p}$, so that ${p\over c(p)} \leq {\lambda_{1,i}\over\sigma^2}, {\lambda_{2,i}\over\sigma^2} \leq p$. We then plot the upper bound in Theorem~\ref{thm:moderateSNR}, against the empirical classification error.}
In the high SNR regime (Fig.~\ref{fig:MAP_Pe_highSNR}), the classification error decays faster in Case 2 than in Case 1, consistent with Theorem \ref{thm:highSNR}. 
In the low SNR regime (Fig.~\ref{fig:MAP_Pe_lowSNR}), there is little difference in classification error between the two cases, consistent with Theorem \ref{thm:lowSNR}. 
\textcolor{\chcolor}{
In the moderate SNR regime (Fig.~\ref{fig:MAP_Pe_lowSNR}), 
classification performance in case 1 is inferior to that in case2, because there is a shared 1-dimensional subspace, and this is predicted by Theorem~\ref{thm:moderateSNR}.}
\begin{figure}[h!]
	\vspace{-20pt}
	\subfloat[high SNR regime]{\label{fig:MAP_Pe_highSNR}\includegraphics[width=0.33\columnwidth]{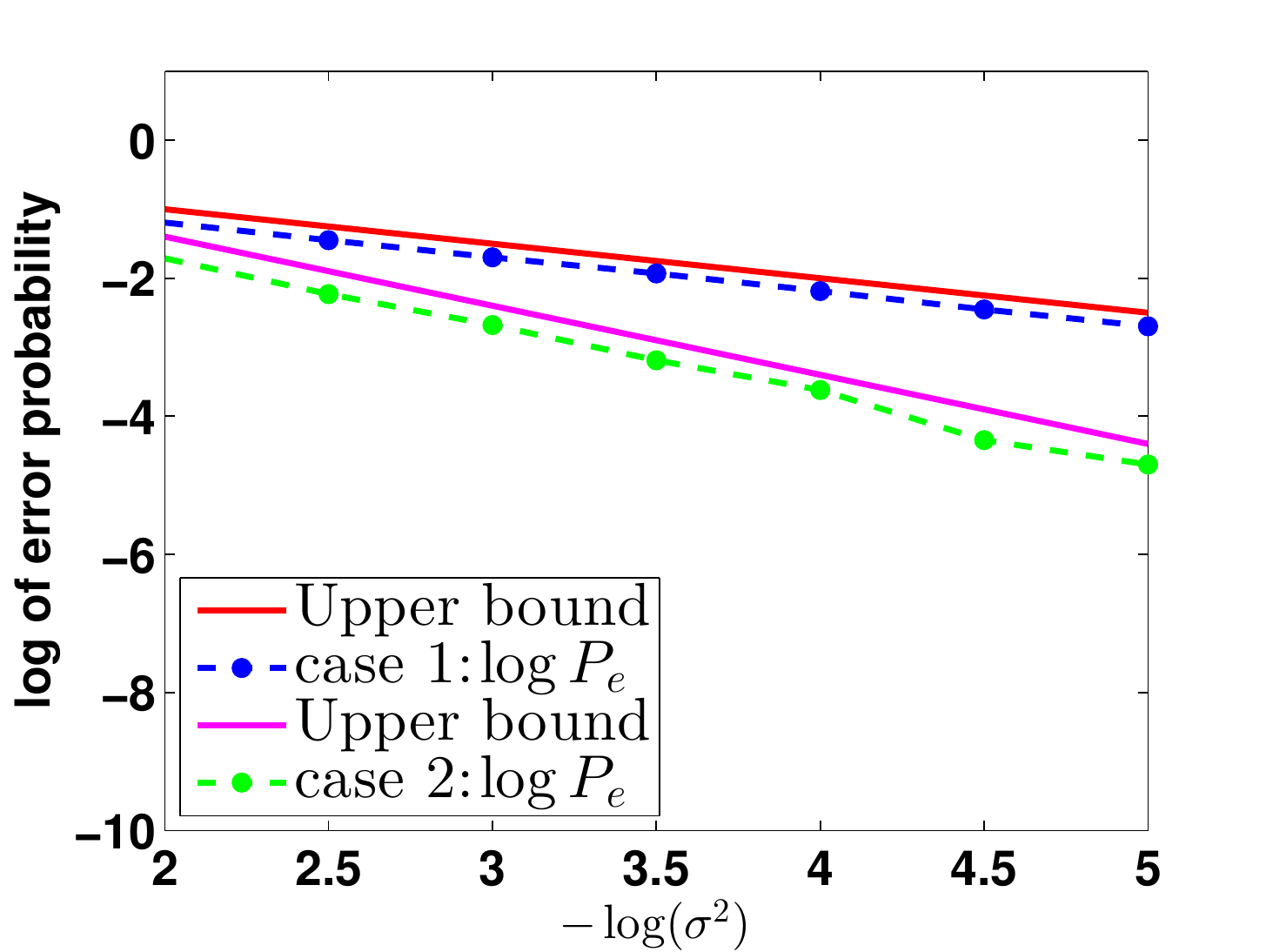}}
	\subfloat[low SNR regime]{\label{fig:MAP_Pe_lowSNR}\includegraphics[width=0.33\columnwidth]{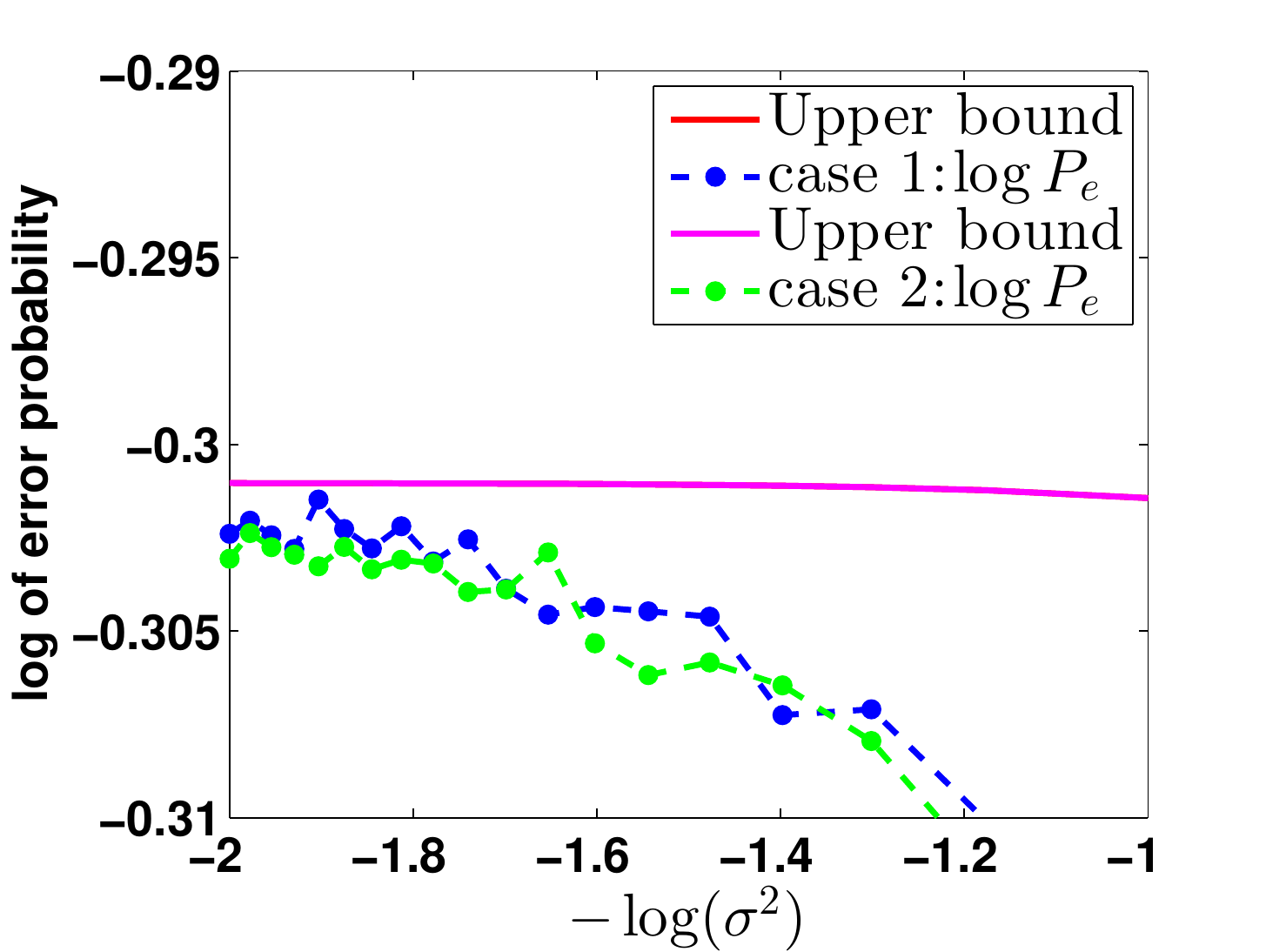}}
	\subfloat[\textcolor{\chcolor}{moderate SNR regime}]{\label{fig:MAP_Pe_moderateSNR}\includegraphics[width=0.33\columnwidth]{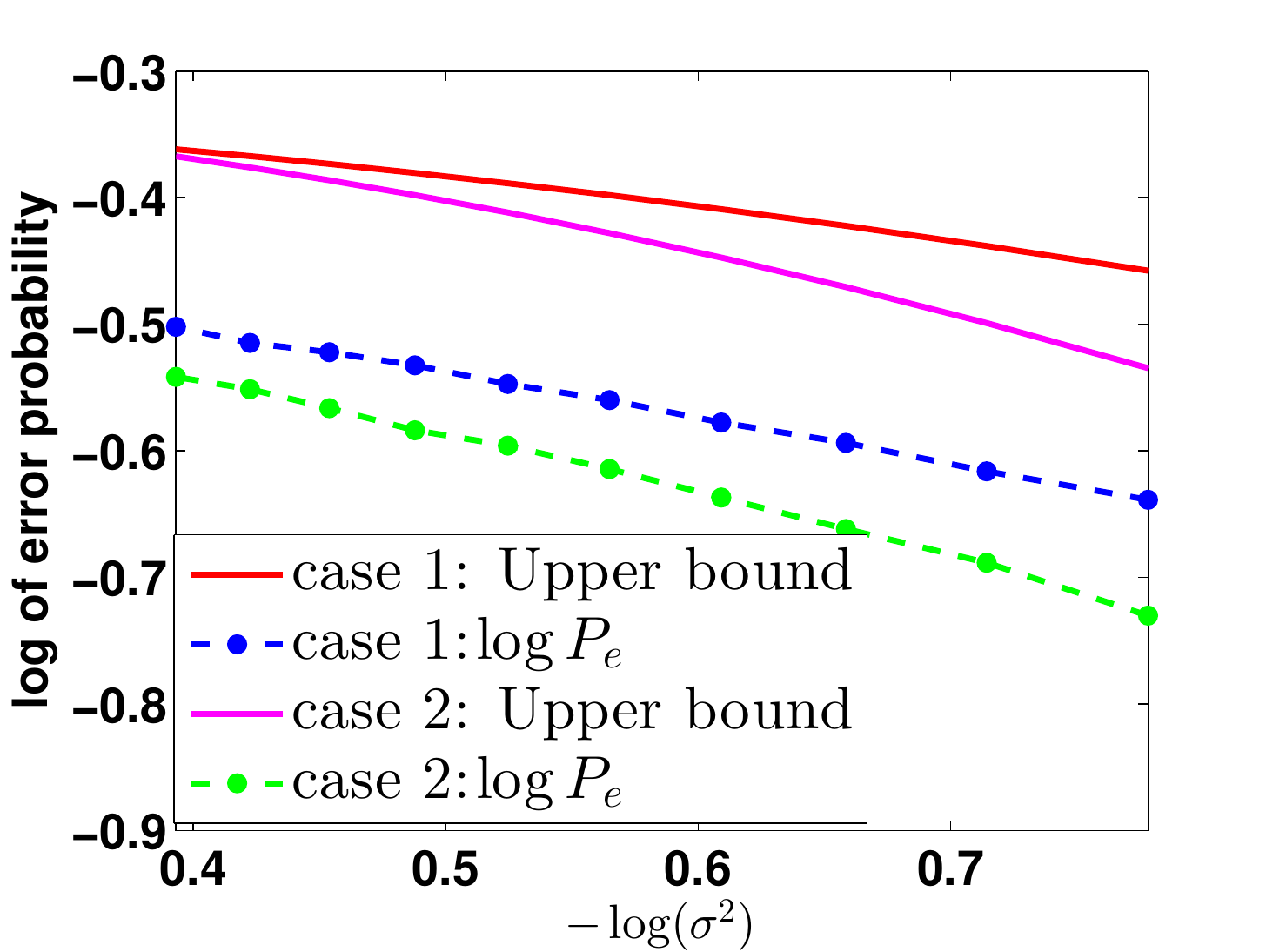}}
	\caption{
		 Error probability as a function of the degree of mismatch. Dashed lines represent empirical estimates, and solid lines represent upper bounds. In the low SNR regime the two upper bounds coincide.
		}
	\label{fig:MAP_Pe}
	\vspace{-10pt}
\end{figure}

Concluding this section, we have characterized the pair-wise classification error using the principal angles between a pair of subspaces. The union bound then makes it possible to derive an upper bound on classification error for multiple classes.

%========================================================================================
\section{Nearest Subspace Classifier: extending GMM}
\label{sec:NSC}
%========================================================================================
If the class distribution is known (for example through its covariance) then the MAP classifier is optimal. 
If however we only know that each class is near a known low-dimensional subspace (possibly inferred from less training data) then we can substitute a Nearest Subspace Classifier (NSC) for the MAP. This Section connects performance of the NSC with principal angles, and for simplicity we focus on discriminating pairs of classes, given that the extension to multiple classes is straightforward.

\textcolor{\chcolor}{
Consider two classes, labeled $C_1$ and $C_2$, distributed near two subspaces with orthonormal bases $\U_1,\U_2\in\R^{n\times d}$. 
%Let the corresponding class label be $\C_1$ and $\C_2$.
The NSC determines the class label of a test sample $\x$, $\hat\C$, by comparing the norms of the projections onto $\U_1$ and $\U_2$.
\ben
\hat\C =\left\{\begin{array}{ll}  
	\C_1 & \|\U_1^\top \x\|^2\geq \|\U_2^\top \x\|^2\\
	\C_2 & \mbox{otherwise}
\end{array}\right..
\een
%\textcolor{red}{It asserts that the inferred class has its basis better aligned with the signal.}
The preferred class label has a basis that is better aligned to the signal.}

%~~~~~~~~~~~~~~~~~~~~~~~~~~~~~~~~~~~~~~~~~
\subsection{Derivation of the Upper Bound}
%~~~~~~~~~~~~~~~~~~~~~~~~~~~~~~~~~~~~~~~~~
Starting from the projection onto each subspace, we model the distribution of these two classes as
\ben
\begin{aligned} 
p(\x|\C_1)=&\int p(\x|\boldsymbol\alpha,\C_1)p(\boldsymbol\alpha)d\alpha= \int \N(\x;\U_1\boldsymbol\alpha,\sigma^2\I)p(\boldsymbol\alpha)d\boldsymbol\alpha\\
p(\x|\C_2)=&\int p(\x|\boldsymbol\alpha,\C_2)q(\boldsymbol\alpha)d\alpha= \int \N(\x;\U_2\boldsymbol\alpha,\sigma^2\I)q(\boldsymbol\alpha)d\boldsymbol\alpha.
\end{aligned}
\label{eq:NSC_binary_setup}
\een
The NSC knows $\U_1$ and $\U_2$, but is blind to $p(\boldsymbol\alpha)$ and $q(\boldsymbol\alpha)$, where $\boldsymbol\alpha$ is the expansion of the projection $\U_i^\top \x$ in the basis $\U_i$. 
%Note that less training data is required to estimate $\U_1$ and $\U_2$ than is required to estimate $\boldsymbol\Sigma_1$ and $\boldsymbol\Sigma_2$. 
Note that since we are not assuming a GMM, the vector $\boldsymbol\alpha$ need not be multivariate normal.

Let $\V\diag\{\cos\theta_1,\dots,\cos\theta_d\}\bf W^\top$ be the singular value decomposition of $\U_1^\top \U_2$, where $\V$, $\bf W$ are unitary, and the principal angles $\{\theta_1,\dots,\theta_d\}$ are taken in ascending order. We may, absorb $\V$, $\bf W$ into $\U_1$, $\U_2$ at the cost of redefining $p(\boldsymbol\alpha)$, $q(\boldsymbol\alpha)$. Thus we may without loss of generality
assume $\V=\bf W=\I$, \textit{i.e.},
\ben
\U_1^\top \U_2 = \diag\{\cos\theta_1,\dots,\cos\theta_d\}\triangleq {\bf C}.
\een

Define $\Pr(\C_2|\C_1)$ as the probability of mistaking $\C_2$ for $\C_1$ and define $\Pr(\C_1|\C_2)$ similarly.
Then the classification error is
\ben
P_e = {1\over2}\Pr(\C_2|\C_1)+{1\over2}\Pr(\C_1|\C_2).
\een
We bound $\Pr(\C_2|\C_1)$ using principal angles, and $\Pr(\C_1|\C_2)$ can be analyzed in the same manner.
We expand $\Pr(\C_2|\C_1)$ using Bayes rule as
\ben
{\Pr}(\C_2|\C_1) = \int{\Pr}(\C_2|\C_1,\boldsymbol\alpha) p(\boldsymbol\alpha)d\boldsymbol\alpha.
\een
We bound $\Pr(\C_2|\C_1,\boldsymbol\alpha)$ by writing $\x = \U_1\boldsymbol\alpha + \n$, where the noise $\n\sim\N(0,\sigma^2\I)$.
\ben
\begin{aligned}
{\Pr}(\C_2|\C_1,\boldsymbol\alpha) =& {\Pr} (\|\U_1^\top (\U_1\boldsymbol\alpha+\n)\|^2 \leq \|\U_2^\top (\U_1\boldsymbol\alpha+\n)\|^2) \\
=& {\Pr} (\|\boldsymbol\alpha+\U_1^\top \n\|^2 \leq \|C\boldsymbol\alpha+\U_2^\top \n\|^2),
\end{aligned}
\label{eq:tmp1}
\een
where the probability is taken w.r.t. $\n$.
Denote the $i$-th column in $\U_1$($\U_2$) as $\u_{1,i}$($\u_{2,i}$), and the $i$-th element of $\boldsymbol\alpha$ as $\alpha_i$.
It follows from Eq.~\eqref{eq:tmp1} that
\ben
\begin{aligned}
&{\Pr} (\|\boldsymbol\alpha+\U_1^\top \n\|^2 \leq \|{\bf C}\boldsymbol\alpha+\U_2^\top \n\|^2)\\
&= {\Pr}\left(\sum_i (\alpha_i+\u_{1,i}^\top \n)^2\leq \sum_i (\cos\theta_i\alpha_i+\u_{2,i}^\top \n)^2\right).
\end{aligned}
\label{eq:tmp2}
\een
We now define $a_i\triangleq \alpha_i+\u_{1,i}^\top \n$ and $b_i\triangleq \cos\theta_i\alpha_i+\u_{2,i}^\top \n$.
Then Eq.~\eqref{eq:tmp2} simplifies to
\ben
\begin{aligned}
&{\Pr}\left(\sum_i (\alpha_i+\u_{1,i}^\top \n)^2\leq \sum_i (\cos\theta_i\alpha_i+\u_{2,i}^\top \n)^2\right)\\
&={\Pr}\left(\sum_i(a_i+b_i)(a_i-b_i)\leq 0\right).
\end{aligned}
\een
\begin{lemma} 
	\label{lemma:independence}
	Let $a_i$, $b_i$ as defined as above. 
	For any pair of $i,j$ where $i\neq j$:
	\begin{enumerate}
		\item $a_i$ is independent from $a_j$
		\item $b_i$ is independent from $b_j$
		\item $a_i$ is independent from $b_j$
		\item $a_i+b_i$ is independent from $a_i-b_i$
	\end{enumerate}
\end{lemma}
\begin{proof}
	The proof is given in appendix~\ref{sec:proof_NSC}.
\end{proof}
It follows from Lemma \ref{lemma:independence} that $\sum_i(a_i+b_i)(a_i-b_i)$ is the sum of products of independently distributed normal random variables. However the product of independently distributed normal random variables need not be normal, and so we need to show that $(a_i+b_i)(a_i-b_i)$ is normally distributed.

\begin{lemma}[product of normal random variable\cite{Aroian1947}]
	\label{lemma:productNormal}
	Let $x\sim\N(\mu_x,\sigma_x^2)$ and $y\sim\N(\mu_y,\sigma_y^2)$ be two independent normal variables. 
	If $\mu_x/\sigma_x\to\infty$ and $\mu_y/\sigma_y\to\infty$ in any manner, 
	then the distribution of $xy$ approaches normality with mean $\mu_x\mu_y$ and 
	variance $\mu_x^2\sigma_y^2+\mu_y^2\sigma_x^2+\sigma_x^2\sigma_y^2$.
\end{lemma}
Applying Lemma~\ref{lemma:productNormal} and combining the independence stated in Lemma~\ref{lemma:independence}, we have 
\begin{lemma}
	\label{lemma:ourProductNormal}
	As $\sigma\to0$, $\sum_i(a_i+b_i)(a_i-b_i)\sim\N\left(\sum_i \sin^2\theta_i\alpha_i^2, 4\sigma^2\sum_i\sin^2\theta_i (\alpha_i^2+\sigma^2) \right)$	
\end{lemma}
\begin{proof}
	The proof is given in appendix~\ref{sec:proof_NSC}.
\end{proof}
It follows that ${\Pr}\left(\sum_i(a_i+b_i)(a_i-b_i)\leq 0\right)$ is the tail probability of a normal distribution. Applying the standard tail bound, we arrive at the following theorem.
\begin{theorem}
	\label{thm:NSCbound}
	As $\sigma^2\to0$, the classification error is upper bounded as 
	\[
	P_e \leq\int \Er(\boldsymbol\theta,\boldsymbol\alpha,\sigma^2) {p(\boldsymbol\alpha)+q(\boldsymbol\alpha)\over2} d\alpha
	\]
	where $\Er(\boldsymbol\theta,\boldsymbol\alpha,\sigma^2)={1\over2}\exp\left[-\frac{\left(\sum_{i=1}^d \sin^2\theta_i\alpha_i^2 \right)^2}{8\sigma^2\sum_{i=1}^d\sin^2\theta_i (\alpha_i^2+\sigma^2 )} \right]$.
\end{theorem}
\begin{proof}
	The proof is given in appendix~\ref{sec:proof_NSC}.
\end{proof}
\begin{figure}[h!]
	\vspace{-25pt}
	\subfloat[case 1]{\label{fig:contour1}\includegraphics[width=0.5\columnwidth]{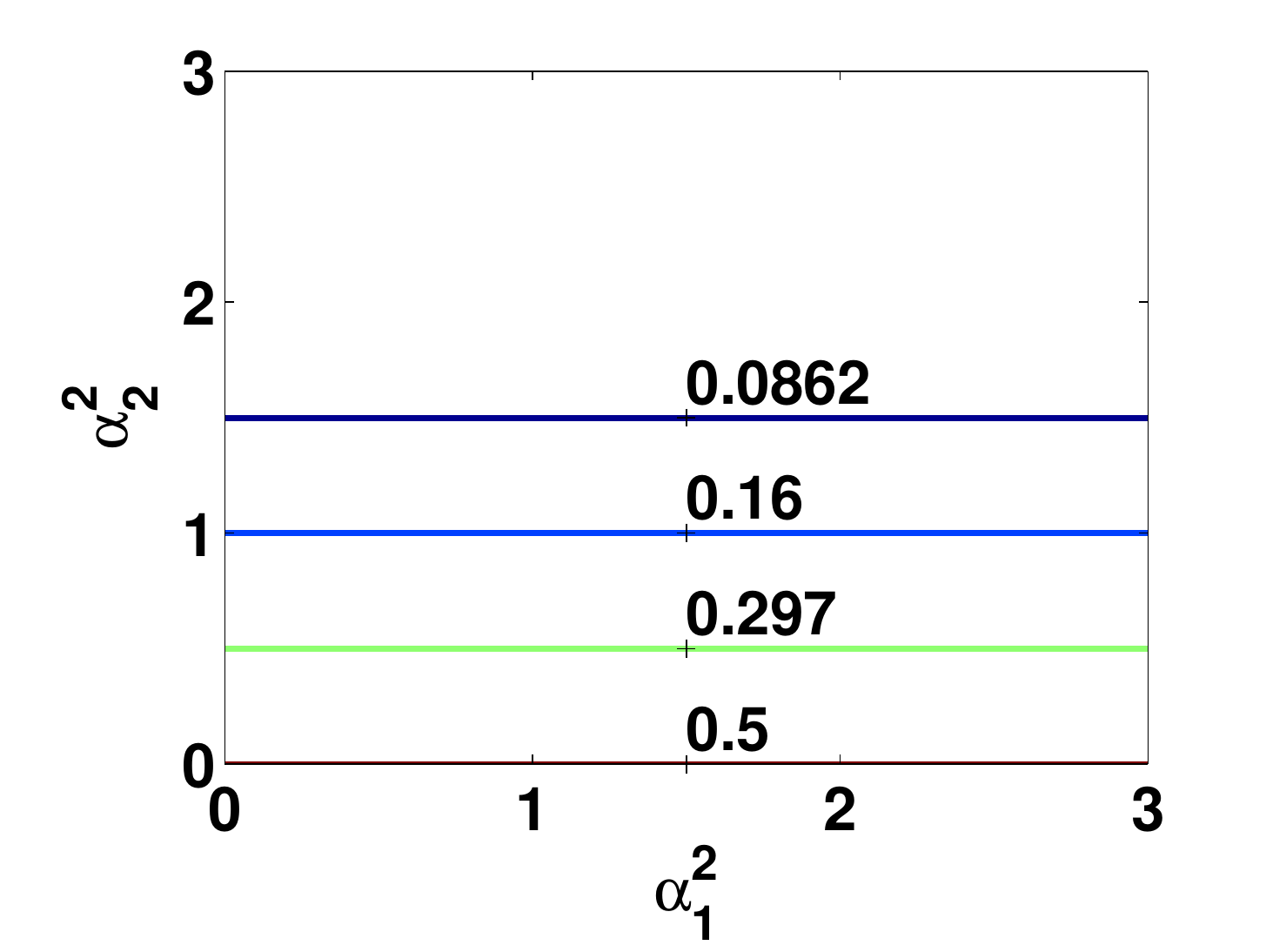}}
	\subfloat[case 2]{\label{fig:contour2}\includegraphics[width=0.5\columnwidth]{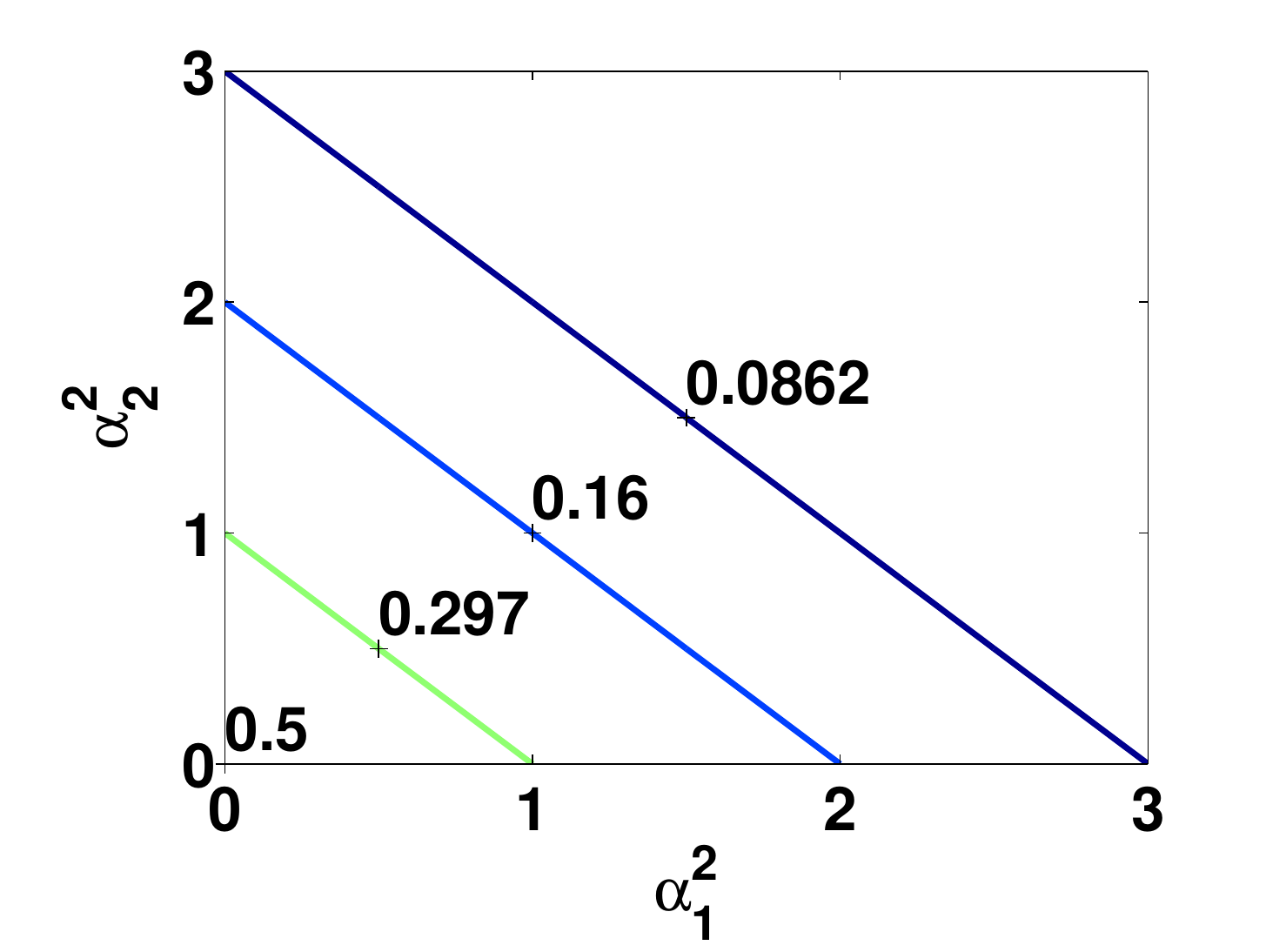}}
	\caption{Lines on which $\Er$ is constant for the two case studies introduced in section~\ref{sec:MAP_nume}.}
	\label{fig:contour}
\end{figure}
We return to the two case studies introduced in Section \ref{sec:MAP_nume} to provide some intuition about the kernel $\Er$. The principal angles are $[0, \pi/2]$ in Case 1, and $[\pi/4,\pi/4]$ in Case 2. In Case 1, the kernel is constant on horizontal lines, and in Case 2, it is constant on lines of slope -1. These two cases are shown in Fig. \ref{fig:contour}, and we now make a number of general observations.
\begin{remark}
	\label{remark:theoremNSC}
	1. 
		$\Er(\boldsymbol\theta,\boldsymbol\alpha,\sigma^2)$ is monotonically decreasing w.r.t. $\sum_i \sin^2\theta_i\alpha_i^2$, 
		and monotonically increasing w.r.t. $\sigma^2$.
		Therefore, bigger principal angles or signal energy results in smaller classification error.
		Bigger noise results in bigger classification error.
	2.
		Ignoring the higher order term of $\sigma^2$ in the denominator inside the $\exp(\cdot)$, we have 
		\[\begin{aligned}
		\Er(\boldsymbol\theta,\boldsymbol\alpha,\sigma^2)
		&\approx{1\over2}\exp\left(-\frac{\sum_i\sin^2\theta_i\alpha_i^2}{8\sigma^2}  \right)\\
		%&\leq {1\over2} \exp\left(-\sqrt{\left(\sum_i\sin^4\theta_i\right) \left(\sum_i \alpha_i^4\right)} \over 8\sigma^2  %\right)
		\end{aligned}\]
		which clearly indicates that classification performance is a function of discernibility (the sine principal angles) weighted by signal energy (the $\alpha_i^2$'s).
	3.
		For fixed energy, 
		classification error is decreased by allocating larger $\alpha_i^2$ to larger $\theta_i$.
\end{remark}

%~~~~~~~~~~~~~~~~~~~~~~~~~~~~~~~~~~~~~~~~~
\subsection{Numerical Analysis of Synthetic Data}
%~~~~~~~~~~~~~~~~~~~~~~~~~~~~~~~~~~~~~~~~~
We now examine the agreement between empirical error and the upper bound given in Theorem 3. 
Set $n=6$, $d=2$,
$$\U_1=\bmx \I_2, \0_4\emx^\top,~ \U_2=\bmx \cos\theta &0&0&0&\sin\theta&0\\0&\cos\theta &0&0&0&\sin\theta \emx^\top,$$ 
so that the two principal angles between $\U_1$ and $\U_2$ are $\theta_1=\theta_2=\theta$. 
Set $p(\boldsymbol\alpha)=q(\boldsymbol\alpha)=\N(\boldsymbol\alpha;0,\I_2)$, and vary $\sigma^2$ in $[0.01, 0.5]$. 
Fig.~\ref{fig:NSC1} considers three values of $\theta$ ($\pi/6$, $\pi/4$, and $\pi/3$), and shows that empirical NSC classification error tracks the upper bound obtained by numerical integration.
\begin{figure}[h!]
	\vspace{-10pt}
	\centering
	\subfloat[]{\label{fig:NSC1}\includegraphics[width=0.5\columnwidth]{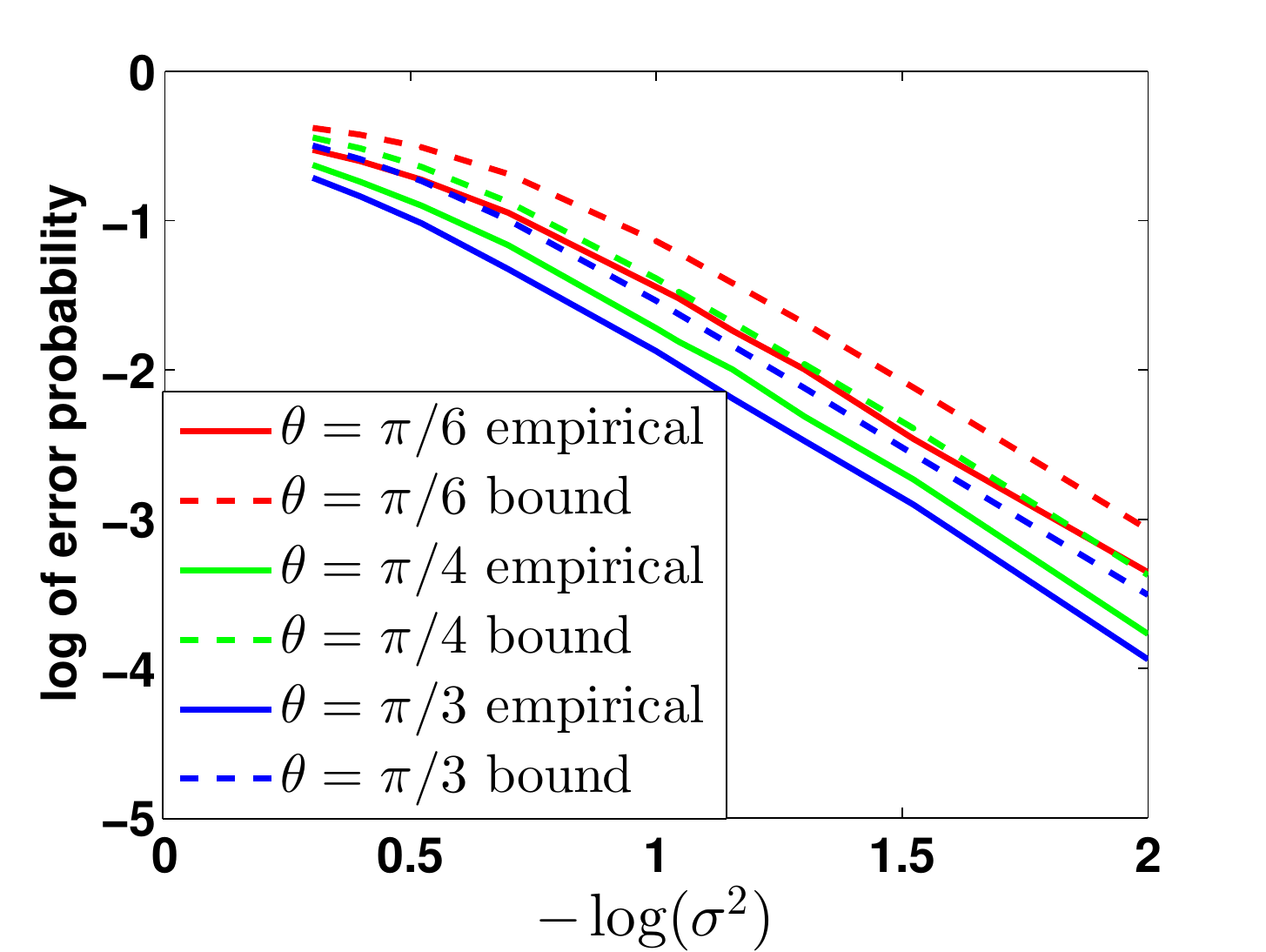}}
	\subfloat[]{\label{fig:NSC2}\includegraphics[width=0.5\columnwidth]{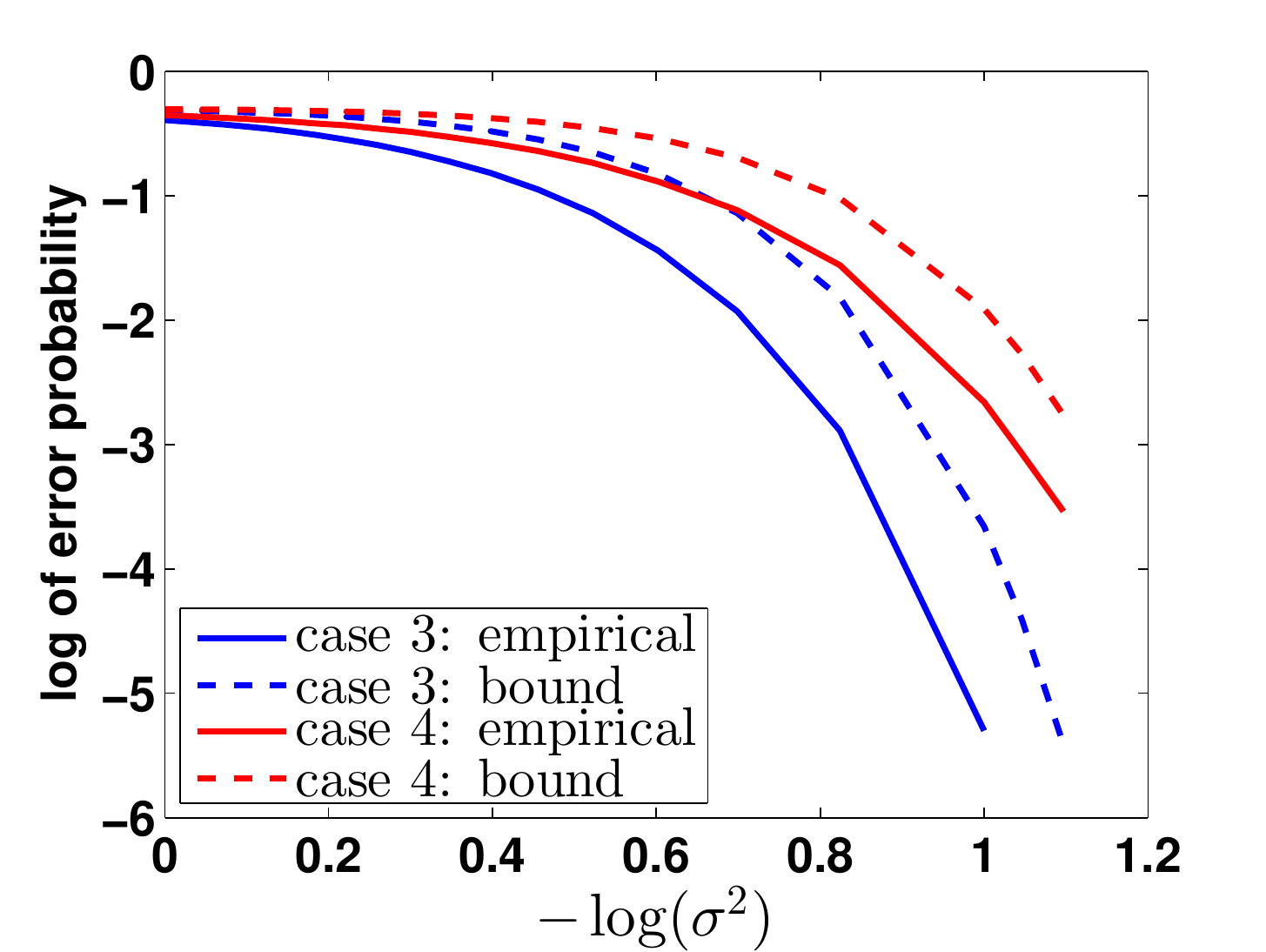}}
	\caption{
		Comparison of empirical NSC classification error with the upper bound obtained by numerical integration. (a) Larger principal angles reduce classification error; (b) Disproportionate assignment of signal energy to larger principal angles reduces classification error.}
\end{figure}

Next we examine the dependence of classification error on distribution of signal energy across the two modes. 
Set $n=6$, $d=2$, $\U_1=\bmx \I_2, \0_4\emx^\top$ and 
$$\U_2=\bmx \cos(\pi/6) &0&0&0&\sin(\pi/6)&0\\0&\sin(\pi/6) &0&0&0&\cos(\pi/6) \emx^\top,$$ 
so that the two principal angles are $\theta_1=\pi/6$ and $\theta_2=\pi/3$. Fix $\|\boldsymbol\alpha\|^2=1$, and compare the case when $\boldsymbol \alpha$ is distributed such that $|\alpha_1|<|\alpha_2|$ (Case 3 in Fig. \ref{fig:NSC2}), with the case when $\boldsymbol \alpha$ is distributed such that $|\alpha_1|>|\alpha_2|$ (Case 4 in Fig. \ref{fig:NSC2}). Empirical error is calculated for a range of noise variances, by randomly drawing 10,000 sample per class. Empirical NSC classification error tracks the upper bound given by numerical integration, with performance of Case 3 superior to that of Case 4.

%========================================================================================
\section{\trait: tunable recognition adapted to intra-class target}
\label{sec:trait}
%========================================================================================
In the previous theorems, it is the principal angles that determine the performance of the classifiers \textcolor{\chcolor}{in different SNR regimes}. 
This suggests that we might improve classification by applying a linear transformation that optimizes principal angles, even at the cost of reducing dimensionality.

We denote the collection of all labeled training samples as
$\X=[\X_1,\dots,\X_K]\in\R^{n\times N}$,
where columns in the submatrix $\X_k\in\R^{n\times N_k}$ are samples from the $k$-th class.
The signal subspace of $\X_k$ is spanned by the orthonormal basis $\U_k$ defined above.
The linear transform $\A\in\R^{m\times n}~(m\leq n)$ is designed to maximize separation of the subspaces $\A\U_1,\dots,\A \U_K$.
The maximal separation is achieved when $(\A\U_j)^\top (\A\U_k)=0$ for all $j\neq k$. 
In this case, all the principal angles are $\pi/2$. 
One approach is to use the SVD to compute the $\U_k$ and then to learn the linear transformation $\A$. However we may avoid pre-computing the $\U_k$ by simply encouraging
$(\A\X_j)^\top (\A\X_k)=0$ for all $j\neq k$.

We shall require that the transform $\A$ preserve some specific characteristic or trait of each individual class.
For example, we may target $(\A\X_k)^\top (\A\X_k)=\X_k^\top \X_k$ for all $k$, so that the original intra-class data structure (with noise) is preserved. 
Given access to a denoised signal, $\hat \X_k$, we might instead target $(\A\X_k)^\top (\A\X_k)=\hat \X_k^\top \hat \X_k$ again for all $k$.
In this case, the intra-class dispersion due to noise is suppressed. 
Thus, the Gram matrix $\bf T$ of the transformed signal can be designed to target preservation of particular intra-class structure. 
We formulate the optimization problem as
\ben
\min_{\A\in\R^{m\times n}} {1\over N^2}\|(\A\X)^\top(\A\X)-{\bf T}\|_F^2.
\label{eq:TRAIT}
\een
The block diagonal structure of the target Gram matrix $\bf T$ promotes larger principal angles between subspaces. 
At the same time the diagonal blocks can be tuned to different characteristics of individual classes. 
For example, when side information is available, we may consider incorporating it in diagonal blocks.
Here we only consider
\ben
{\bf T}=\mbox{diag}\{\X_1^\top \X_1,\dots,\X_K^\top \X_K\},
\label{eq:T}
\een
as a proof-of-concept.
We refer to this approach as the TRAIT algorithm,
where the acronym denotes {\bf T}unable {\bf R}ecognition {\bf A}dapted to {\bf I}ntra-class {\bf T}argets.

It is possible to minimize the objective in E.q.~\eqref{eq:TRAIT} by first minimizing $\|\X^\top {\bf P} \X-{\bf T}\|^2$ for ${\bf P}\succeq 0$ (as Proposition~\ref{prop:closedForm}), and then factoring $\bf P$ as ${\bf P}=\A^\top \A$ where $\A\in\R^{m\times n}$. 
\begin{proposition}
	The minimizer of $\|\X^\top {\bf P} \X-{\bf T}\|_F^2$ where ${\bf P}\succeq 0$, is ${\bf P}^\star=(\X\X^\top)^{-1} \X{\bf T}\X^\top (\X\X^\top)^{-1}$.
	\label{prop:closedForm}
\end{proposition}
\begin{proof}
	Proof is detailed in appendix~\ref{proof:Prop1}.
\end{proof}
However when $m<n$, such a rank-$m$ decomposition may not exist since this $\bf P$ is not guaranteed to be rank deficient.
An alternative is to learn a rank deficient $\bf P$ by solving
\[
\min_{P\succeq \0}\|\X^\top {\bf P} \X-{\bf T}\|_F^2+\lambda\|{\bf P}\|_\ast,
\]
\textcolor{\chcolor}{where the nuclear norm $\|\bf P\|_\ast$ regularizes the rank of $\bf P$.}
However this approach requires careful tuning of  $\lambda$, and it is computationally more complex since we work with a matrix $\bf P$ larger than $\A$. Given these considerations, we choose to solve~\eqref{eq:TRAIT} using gradient descent as described in Algorithm \ref{algo:trait}.

\begin{algorithm}[h!]
	\caption{\trait for feature extraction}
	\label{algo:trait}
	\begin{algorithmic}[1]
		\REQUIRE	labeled training samples $\X=[\X_1,\dots,\X_K]$, target dimension $m,~(m\leq n)$, target Gram matrix $\bf T$.
		\ENSURE	feature extraction matrix (transform) $\A\in\R^{m\times n}$.
		\STATE 	Initialize $\A=[{\bf e}_1,\dots,{\bf e}_m]^\top$, where ${\bf e}_i$ is the $i$-th standard basis.
		\WHILE{stopping criteria not met}
		\STATE		Compute gradient \[{\bf G}=\A(\X\X^\top \A^\top \A \X\X^\top-\X{\bf T}\X^\top).\]
		\STATE		Choose a positive step-size $\eta$ and take a gradient step
		\[\A\gets \A-\eta {\bf G}.\]
		\ENDWHILE
	\end{algorithmic}
\end{algorithm}

%~~~~~~~~~~~~~~~~~~~~~~~~~~~~~~~~~~~~~~~~~
\subsection{Related Methods}
\label{sec:compare}
%~~~~~~~~~~~~~~~~~~~~~~~~~~~~~~~~~~~~~~~~~
Linear Discriminant Analysis (LDA) is a classical feature extraction method which assumes each class to be Gaussian distributed. 
It achieves better performance on face recognition tasks than does PCA~\cite{FisherFace}.
LDA does not assume near low-rank structure of the covariances, 
and therefore considers a different data geometry than the one here studied.

Methods of feature extraction based on random projection have recently been developed and successfully applied to face recognition \cite{Wright2009}. Random projection is designed to preserve pairwise distances between all data points uniformly across class labels~\cite{JLlemma}.

More recently, the Low-Rank Transform (\lrt) has been proposed as a method of extracting features~\cite{QiangJMLR}.
It enlarges inter-class distance while suppressing intra-class dispersion.
\lrt uses the nuclear norm, $\|\A\X_i\|_\ast$, to measure the dispersion of the (transformed) data.
The transform $\A$ is
\[
\argmin_{\A\in\R^{m\times n}:\|\A\|_2\leq c} \sum_{i=1}^K\|\A\X_i\|_\ast-\|\A\X\|_\ast.\\
\label{eq:LRT}
\]
%
%In terms of computation efficiency, \lrt needs the SVD of all the $\A\X_i$'s and $\A\X$ to compute the sub-gradient for each iteration. This results in a cubic complexity w.r.t. the number of training samples. In contrast, each iterate of TRAIT has quadratic complexity. 
\begin{comment}
In high SNR regime, theorem~\ref{thm:highSNR} suggests that a big rank of the union of subspaces (classes) would decrease classification error faster w.r.t. vanishing noise. LRT encourages the rank of the union to be big. 
In this sense, we can consider LRT as working in a regime where model mismatch is small. In contrast, TRAIT does not specify the SNR regime and it may potentially be more robust to model mismatch (as validated by experiments in section~\ref{sec:robustTRAIT}).
\end{comment}
\textcolor{\chcolor}{
What motivates the choice of the nuclear norm is that it is the convex relaxation of rank~\cite{QiangJMLR}. 
In the high SNR regime, Theorem~\ref{thm:highSNR} suggests that classification error decreases when the union of subspaces has large rank. LRT encourages the rank of the union to be large, and it works well in a regime where model mismatch is small. Experiments presented in Section~\ref{sec:robustTRAIT} suggest that TRAIT may be more robust to model mismatch (Fig.~\ref{fig:varySNR_synthetic}).
}
%~~~~~~~~~~~~~~~~~~~~~~~~~~~~~~~~~~~~~~~~~~~~~~~~~~~~~~~~
\subsection{\textcolor{\chcolor}{Two Properties of the TRAIT Transform}}
%~~~~~~~~~~~~~~~~~~~~~~~~~~~~~~~~~~~~~~~~~~~~~~~~~~~~~~~~
{\color{\chcolor}
On synthetic and measured data,
we show that TRAIT effectively enlarges the angles between different subspaces and preserves intra-class structure.
We also compare the classification accuracy of features extracted by TRAIT and the methods in Section~\ref{sec:compare}. }
For synthetic data, the class distribution is known exactly, and the MAP classifier is used to measure classification accuracy. For measured data, the class distribution is unknown a priori, and the NSC classifier is employed instead.

%~~~~~~~~~~~~~~~~~~~~~~~~~~~~~~~~~~~~~~~~~
\subsubsection{\textcolor{\chcolor}{Enlargement of the Principal angles}}
%~~~~~~~~~~~~~~~~~~~~~~~~~~~~~~~~~~~~~~~~~
The synthetic dataset has parameters $n=10$, $d=1$ and $K=3$.
\[\boldsymbol\Sigma_k=\U_k\U_k^\top+10^{-2}\I (k=1,2,3),\]
where $\U_k$ is a normalized $n$-vector with i.i.d. Gaussian random entries.
Samples of the $k$-th class are i.i.d drawn from $\N(\0,\boldsymbol\Sigma_k)$. 
For each class, $100$ samples are used for learning the transform and $10000$ are used for testing.
On the training data, we learn the transform respectively via LDA, \lrt, and \trait with target dimension $m=3,\dots,10$. 
Then on each test datum, we apply the learned transforms as well as random projection (each entry drawn from $\N(0,1)$) 
and classify using a MAP classifier.

We visualize original and transformed data via projection (PCA basis) into 3-dimensional Euclidean space. When the target feature dimension $m=3$, the results are shown in Fig. \ref{fig:embed}. 
Each class is represented by a different color.
%and the performance of each method is represented by the angles between the three subspaces. 
\begin{comment}
On the transformed data, we compute the basis (a vector in this $d=1$ case) of each class via SVD. Then we compute the pair-wise angles between these three vectors. The three angles are significantly increased by both TRAIT and LRT, indicating that they both effectively increase separation between the subspaces. While LDA and random projection do not increase separation.
\end{comment}
\textcolor{\chcolor}{
After transforming the data, we use the SVD to calculate the basis vector ($d=1$) that best describes each class, and we calculate the pairwise angles between basis vectors. The pairwise angles are significantly increased by both LRT and TRAIT. By contrast, neither LDA nor random projection increase separation between one-dimensional subspaces.}

\begin{figure}[h!]
	\subfloat[Original, angles:\newline $66.1^\circ$, $21.8^\circ$, $70.0^\circ$]{\includegraphics[width=0.33\columnwidth]{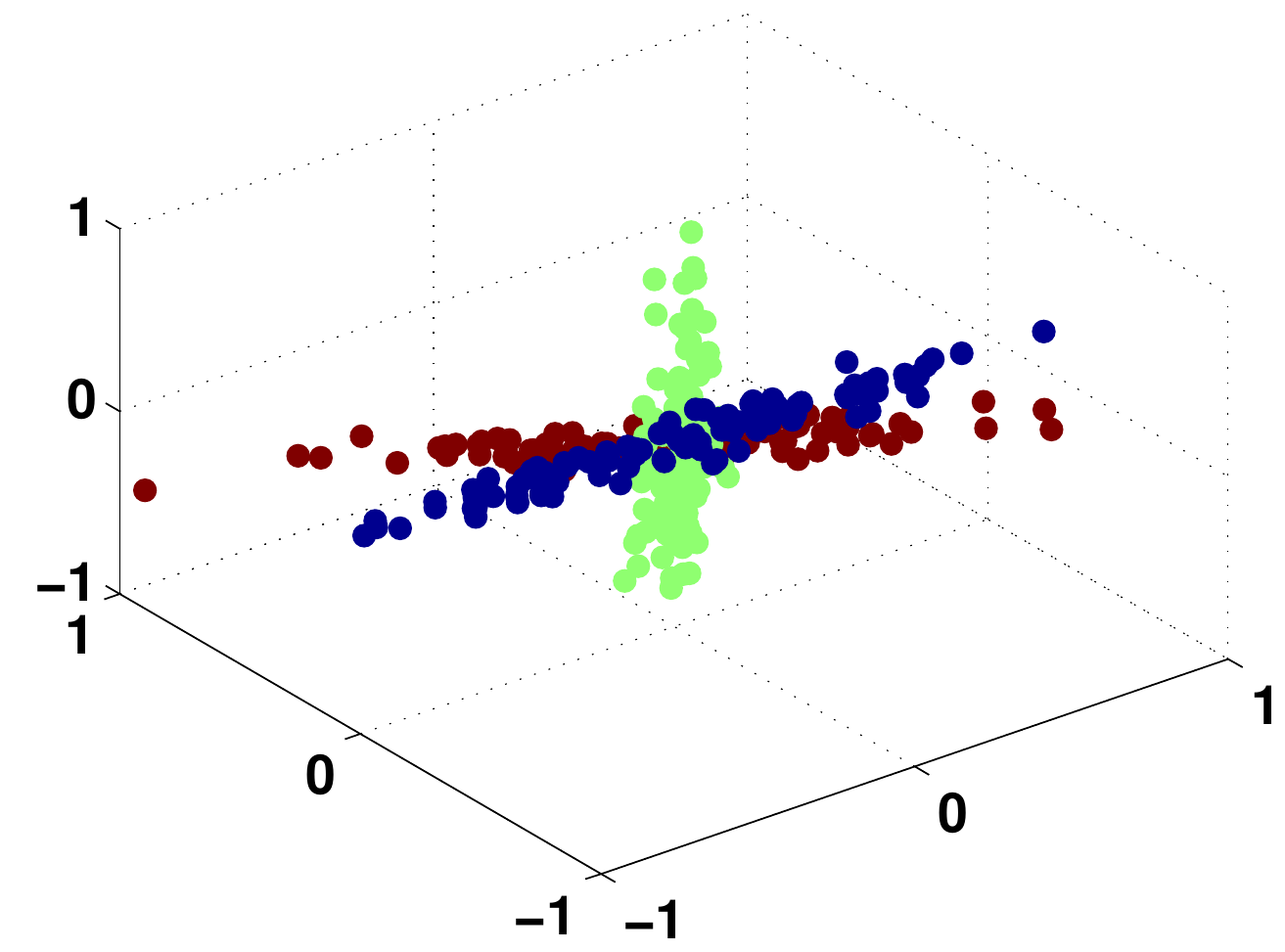}}~~
	\subfloat[\trait, angles:\newline $87.8^\circ$, $72.9^\circ$, $88.7^\circ$]{\includegraphics[width=0.33\columnwidth]{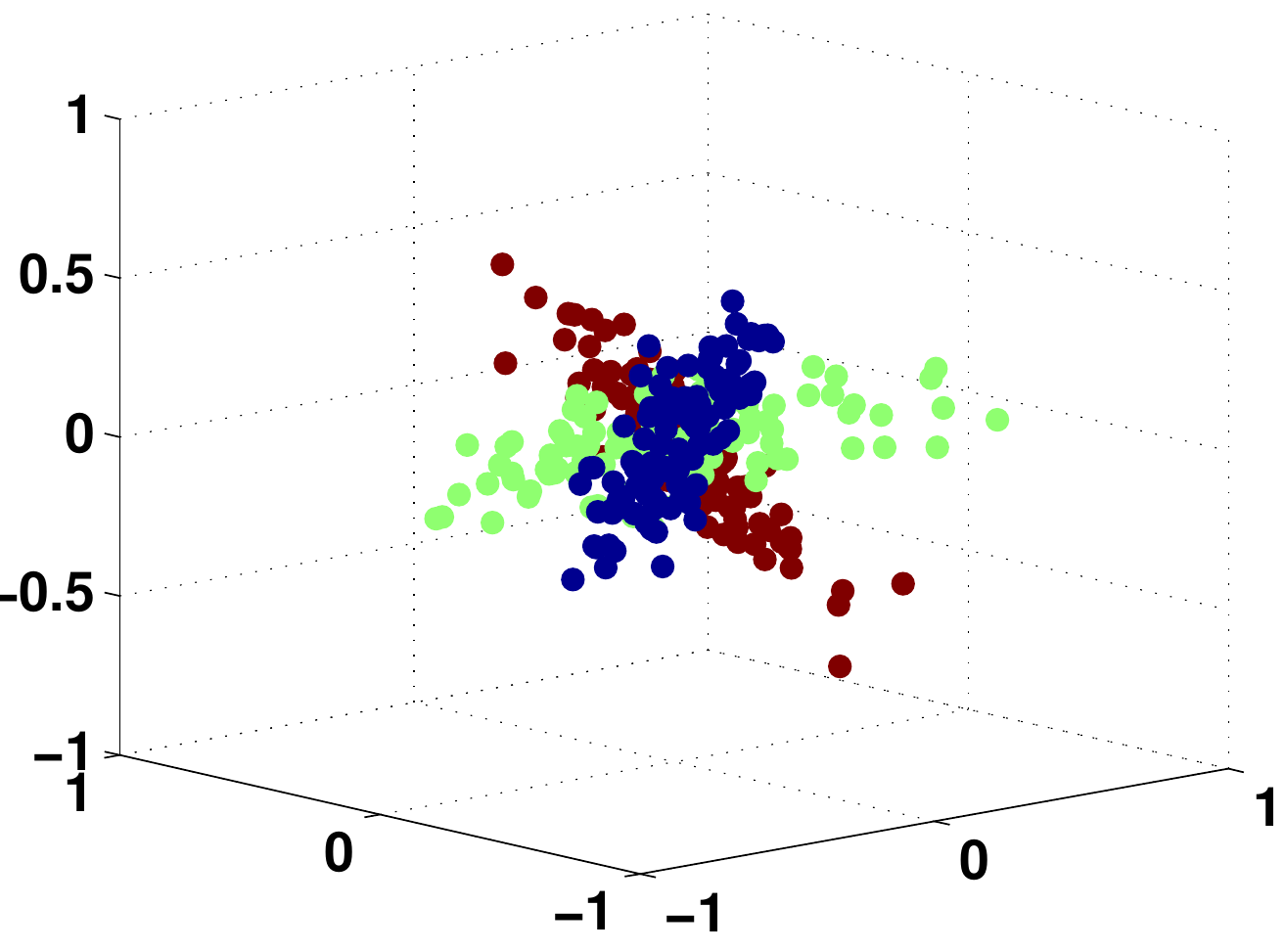}}~~
	\subfloat[\lrt, angles:\newline $86.0^\circ$, $77.1^\circ$, $83.0^\circ$]
	{\includegraphics[width=0.33\columnwidth]{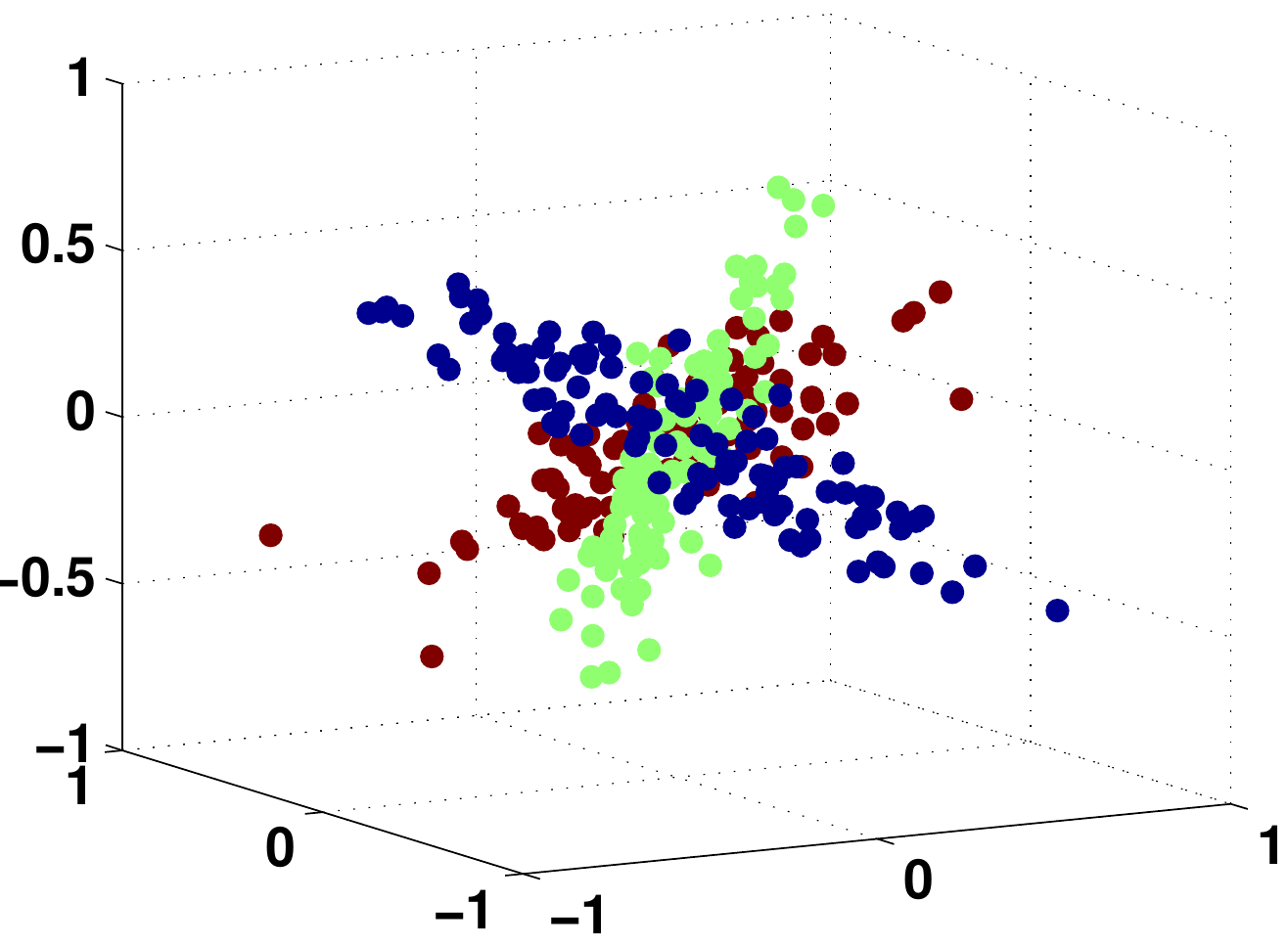}}\\
	\subfloat[LDA, angles:\newline $47.0^\circ$, $86.4^\circ$, $87.8^\circ$]
	{\includegraphics[width=0.33\columnwidth]{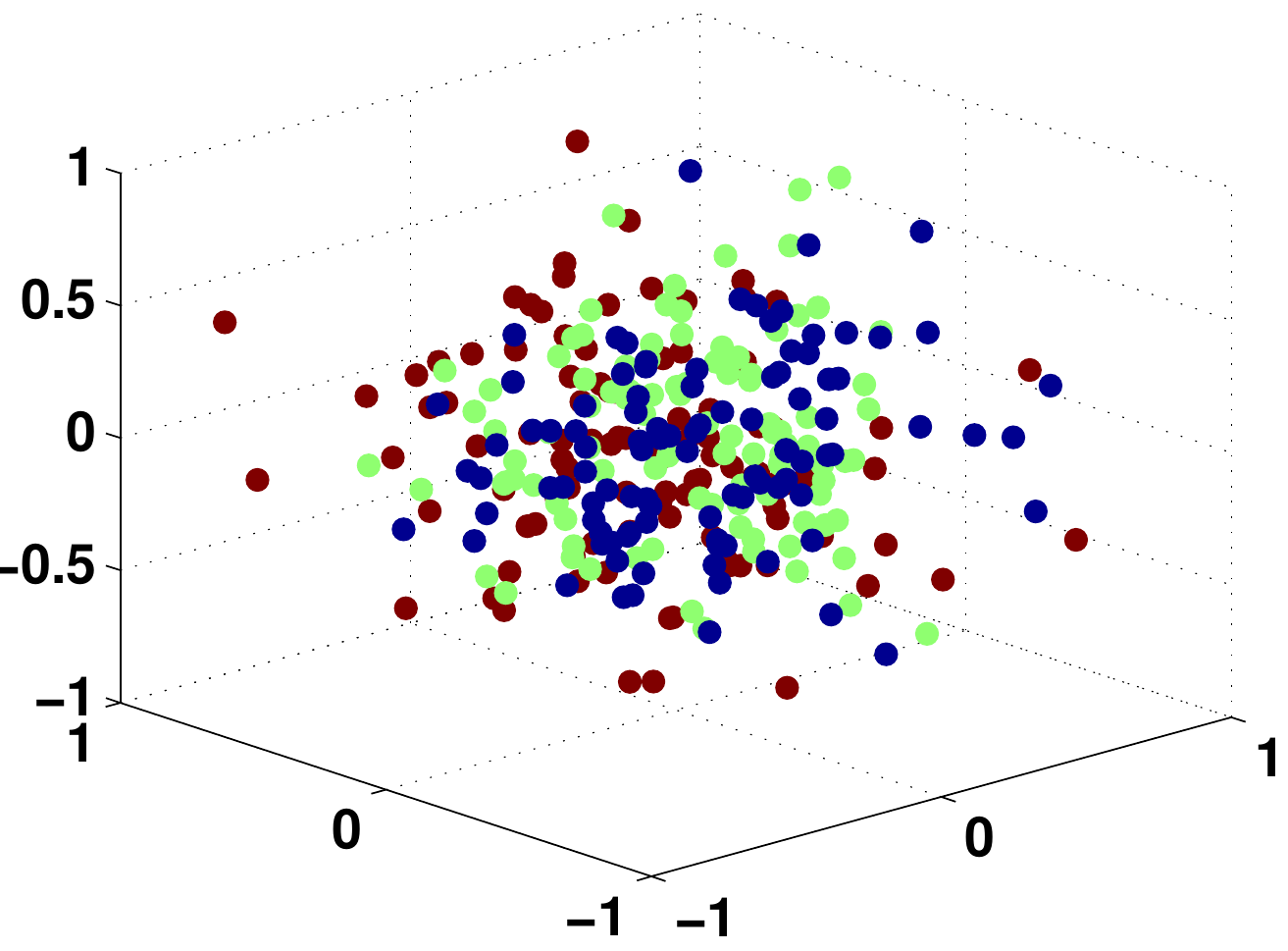}}~~
	\subfloat[Random, angles:\newline $72.4^\circ$, $2.5^\circ$, $73.0^\circ$]{\includegraphics[width=0.33\columnwidth]{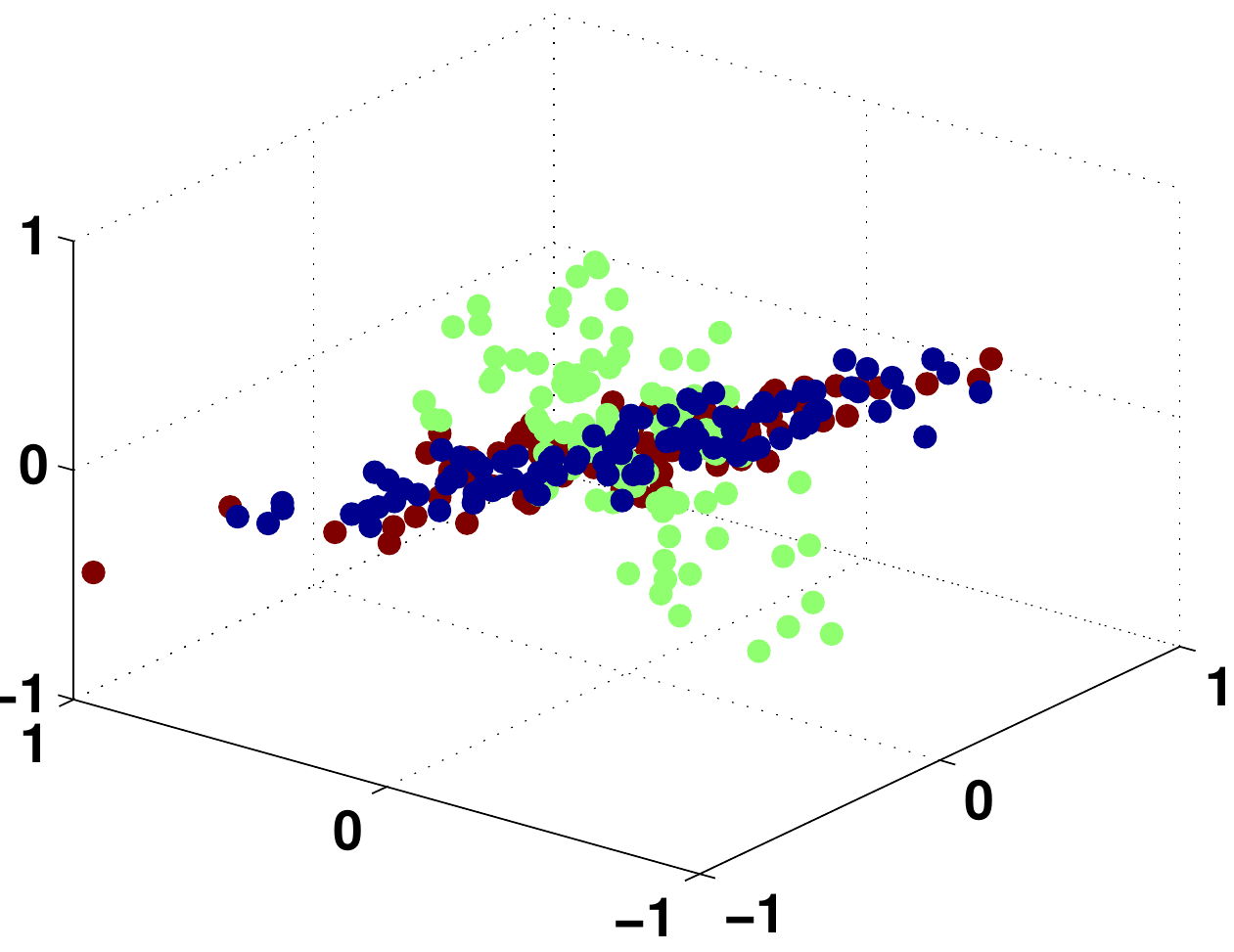}}
	\caption{Embeddings of original and transformed data.}
	\label{fig:embed}
\end{figure}
We now vary the feature dimension $m$, and compare the error probability of the MAP classifier across the different methods of extracting features. Fig. \ref{fig:comparePe} shows that the performance of TRAIT and LRT are similar, and that both are superior to LDA and random projection. Note that after dimension reduction. TRAIT is still able to match error probabilities achieved with the original data.

%~~~~~~~~~~~~~~~~~~~~~~~~~~~~~~~~~~~~~~~~~
\subsubsection{\textcolor{\chcolor}{Preservation of Intra-class Structure}}
\label{sec:yaleB}
%~~~~~~~~~~~~~~~~~~~~~~~~~~~~~~~~~~~~~~~~~
When a convex body, e.g., human face, is illuminated, the resulting image is represented by spherical harmonics. 
It has been shown that a 9-dimensional subspace is sufficient to capture the geometry of an individual subject~\cite{Basri2003}. The extended Yale B face database includes 38 subjects, each with 64 images taken under different illumination conditions. We use a cropped version of this data set\footnote{http://www.cad.zju.edu.cn/home/dengcai/Data/FaceData.html}, where each image is of size $32\times32=1024$.

For each subject, we randomly select half of the 64 images for training, and retain the other half for testing. For all feature extraction methods, we vary the target dimension $m$, and apply the NSC to the transformed data. The NSC achieves much higher accuracy on features extracted by TRAIT and LRT (Fig.~\ref{fig:facePe}).

\begin{figure}[h!]
	\centering
	\begin{minipage}[t]{0.45\columnwidth}
		\centering
		\includegraphics[width=\columnwidth]{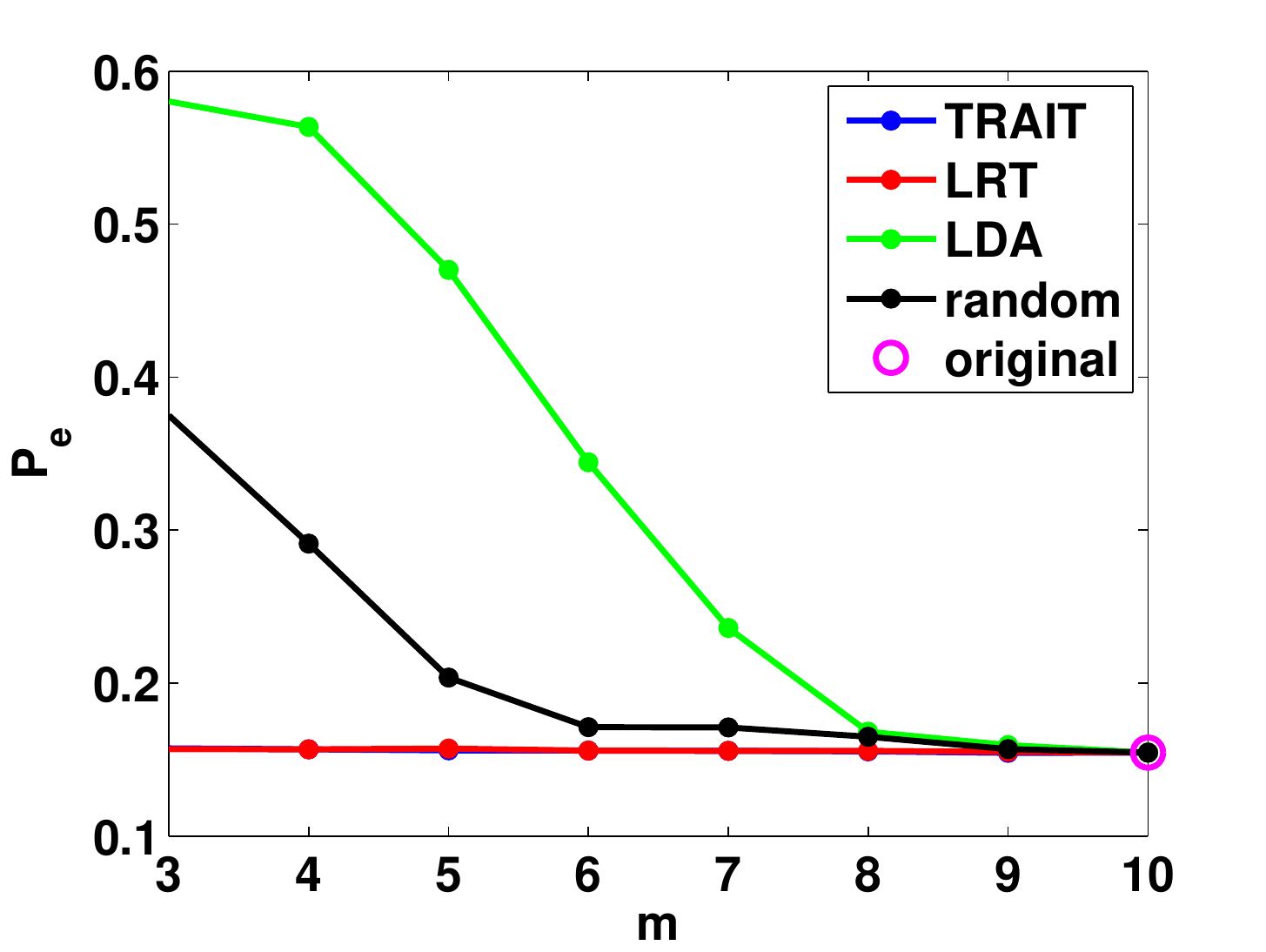}
		\caption{MAP classifier's $P_e$ on transformed data. Note that \trait(blue) and \lrt(red) almost overlap.}
		\label{fig:comparePe}
	\end{minipage}
	\quad
	\begin{minipage}[t]{0.45\columnwidth}
		\centering
		\includegraphics[width=\columnwidth]{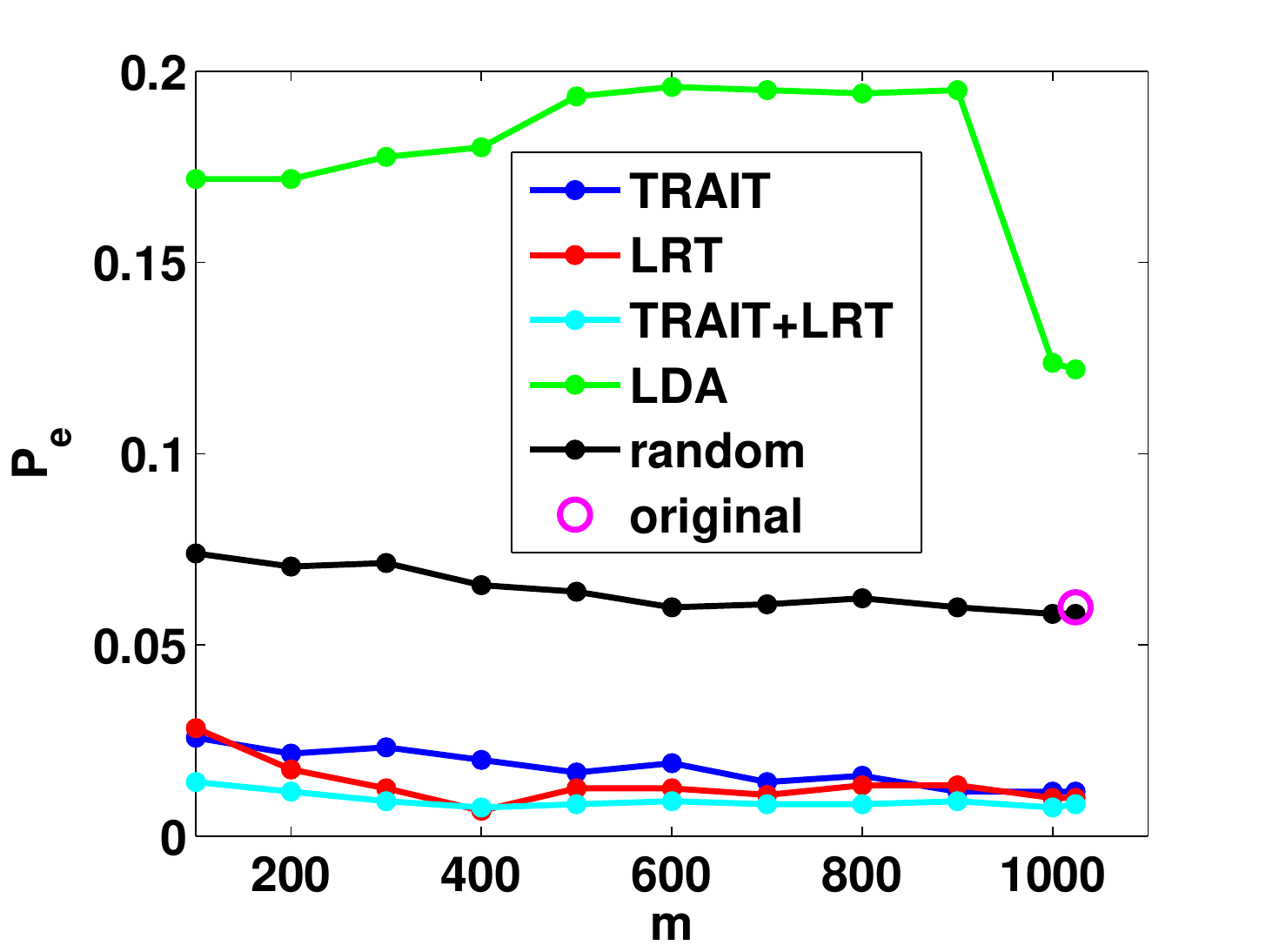}
		\caption{NSC's $P_e$ on original/transformed face images.
			\textcolor{\chcolor}{Concatenation of TRAIT and LRT features (TRAIT+LRT) provides superior results}}
		\label{fig:facePe}
	\end{minipage}
\end{figure}

We also observe in Fig.~\ref{fig:ExampleFace} that the features extracted by TRAIT and LRT are quite different, suggesting that information present in one view is somewhat independent of information present in the other. \textcolor{\chcolor}{This is confirmed by applying NSC to the concatenation of the two views (TRAIT+LRT in Fig.~\ref{fig:facePe})}, and observing that classification accuracy is increased.

The intra-class structure preserving property of TRAIT is evident in Fig.~\ref{fig:ExampleFace} where we view transformed classes as faces in the original image domain. The original images of subject 10 are displayed together with their TRAIT and LRT transforms. TRAIT preserves a diversity of illumination conditions, whereas LRT blurs the differences between images. Classification performance is improved by using LRT and TRAIT features in combination.

\begin{figure}[h!]
	\includegraphics[width=\columnwidth]{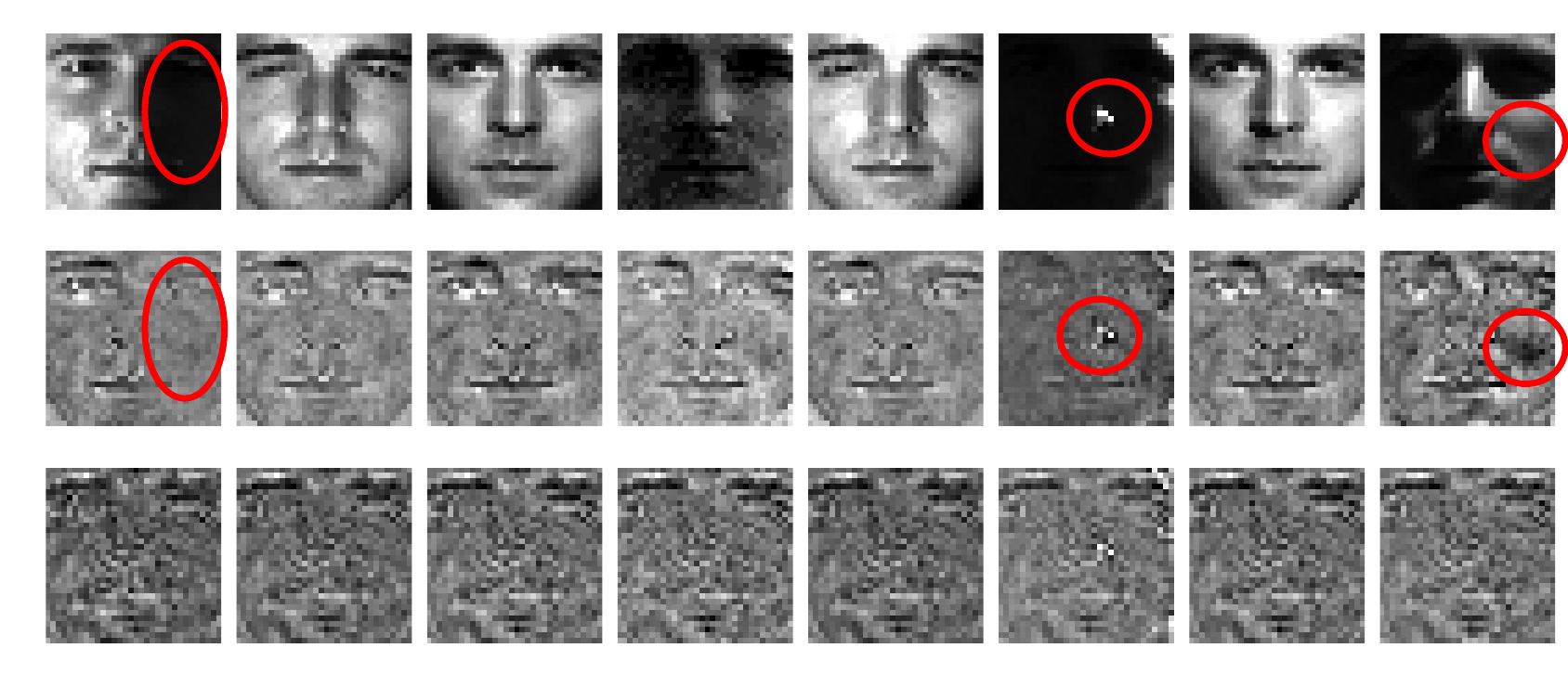}
	\caption{Comparison of original images (top) with TRAIT transformed images (middle) and LRT transformed images (bottom). Red circles indicate structure that is present in both the original and the TRAIT transformed image.}
	\label{fig:ExampleFace}
\end{figure}

%~~~~~~~~~~~~~~~~~~~~~~~~~~~~~~~~~~~~~~~~~
\subsection{\textcolor{\chcolor}{Robustness to Model Mismatch}}
\label{sec:robustTRAIT}
%~~~~~~~~~~~~~~~~~~~~~~~~~~~~~~~~~~~~~~~~~
{\color{\chcolor}
In the previous sections, we have demonstrated the effectiveness of TRAIT and LRT on both synthetic and real data. In this section, we present experiments showing that TRAIT is more robust with respect to model mismatch than is LRT. 
In many real world problems, data may not be exactly GMM distributed. Even if they are, there may not be sufficient training data to learn the covariances. Therefore, we use NSC throughout this section to assess the discriminability of the extracted features. 
Moreover, having seen the effectiveness of dimension reduction in previous sections, 
we turn to learning dimension reduced features, thereby saving computational cost on measured datasets.
\subsubsection{Synthetic Data}
The synthetic data is a three-class dataset, where datum $\x\in\R^{100}$ in the $k$-th ($k=1,2,3$) class is generated as 
\[\x=\U_k\boldsymbol\alpha+\n,\]
with $\U_k\in\R^{100\times 5}$ and $\U_k^\top\U_k=\I$. $\boldsymbol\alpha\sim\rm{Uniform}[-2,2]$ and $\n\sim\N(0,\sigma^2\I_{100})$. 
Note the data is not GMM distributed.
Each class has $100$ training samples and $10000$ testing samples.
We vary $\sigma^2$ and use NSC to classify TRAIT and LRT extracted features. Here we fix the extracted feature dimension to be $30$.

\begin{figure}
	\begin{minipage}[t]{0.45\columnwidth}
	\includegraphics[width=\columnwidth]{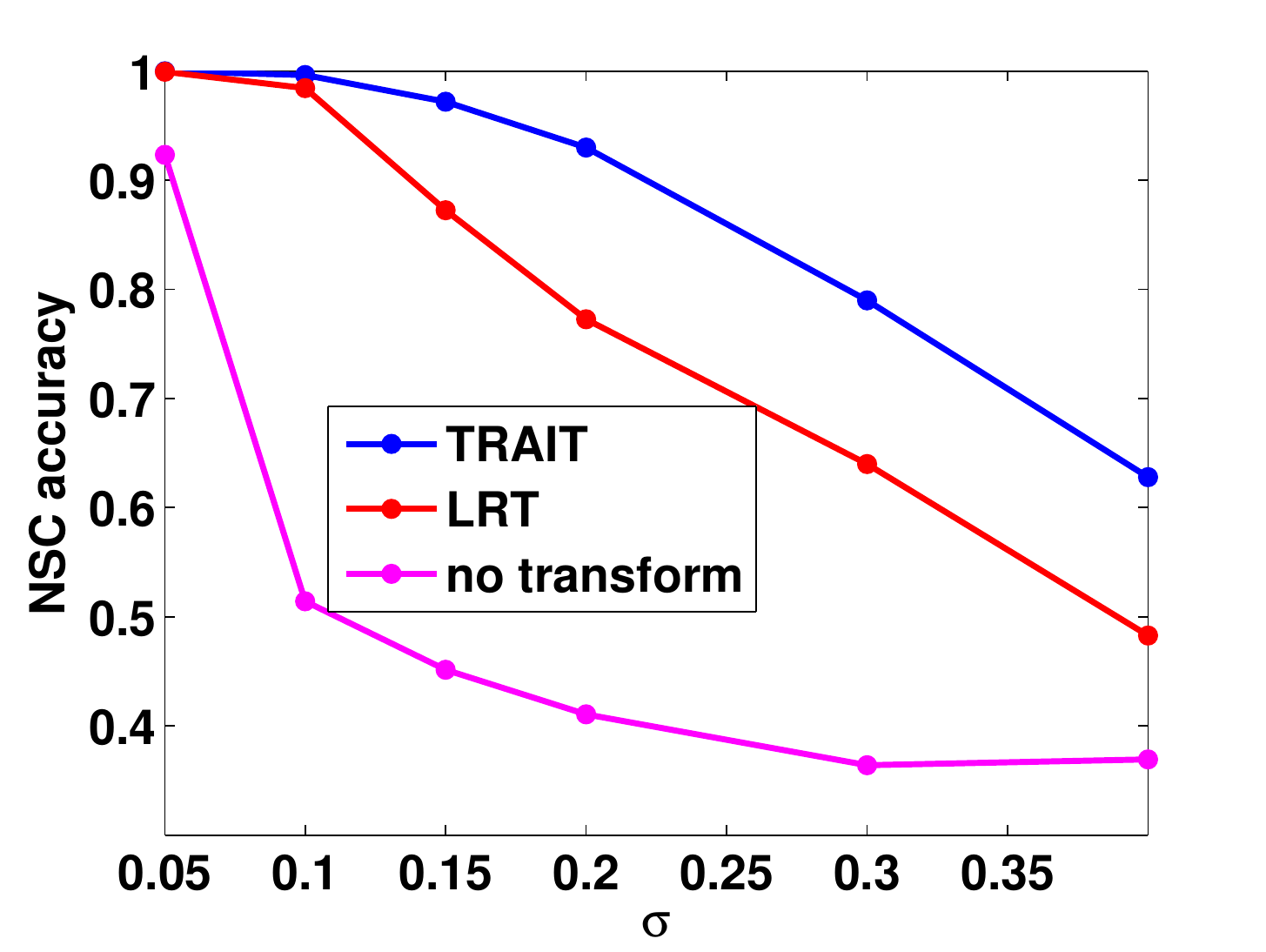}
	\caption{NSC performance on TRAIT and LRT features under different SNR}
	\label{fig:varySNR_synthetic}
\end{minipage}
\quad
\begin{minipage}[t]{0.45\columnwidth}
	\includegraphics[width=\columnwidth]{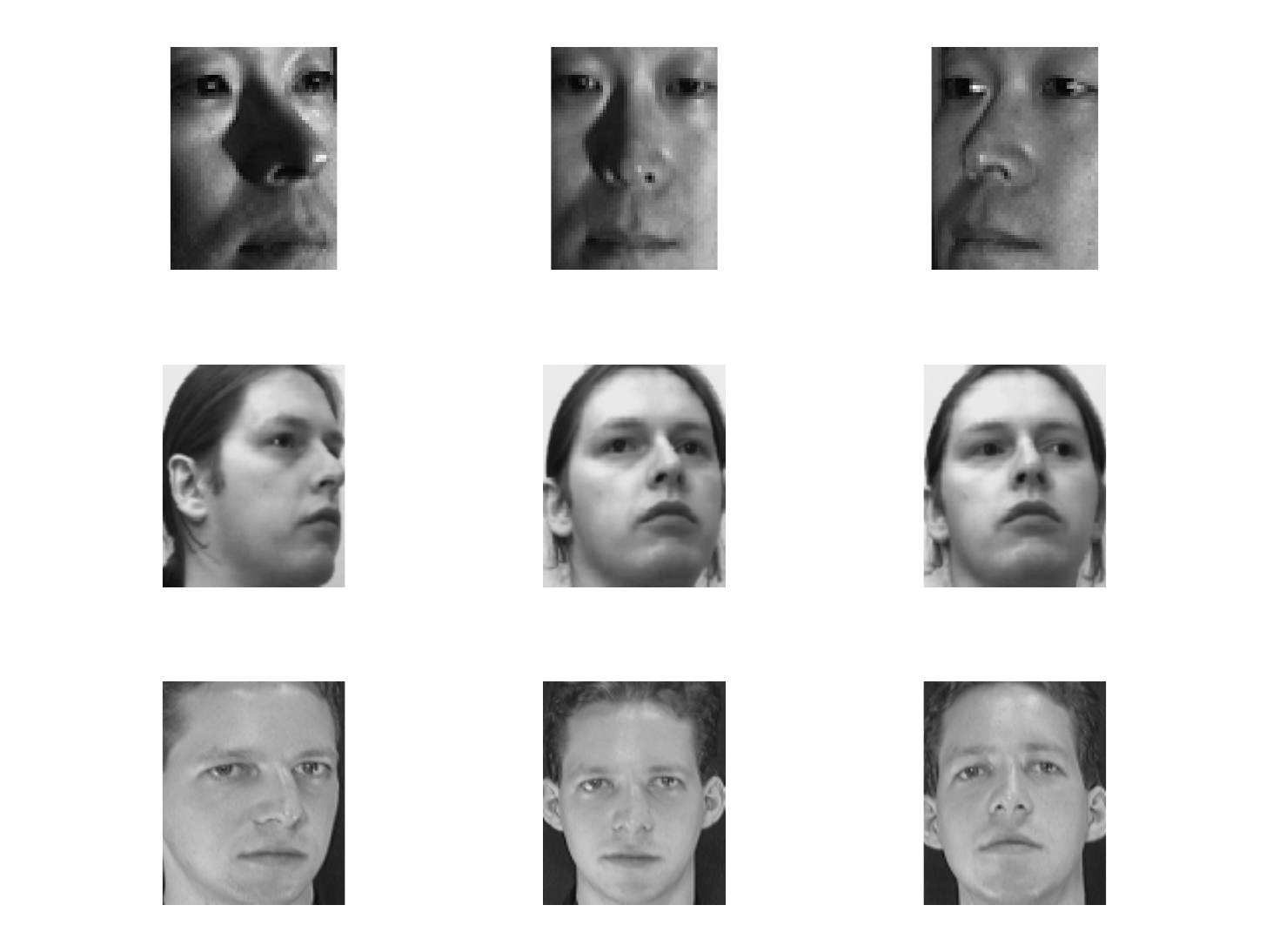}
	\caption{From top to bottom row: 
	subjects in PIE, UMIST and ORL database, taken under different poses}
	\label{fig:PIE}
\end{minipage}
\end{figure}

Fig.~\ref{fig:varySNR_synthetic} shows the NSC classification accuracy as a function of $\sigma$. Both TRAIT and LRT significantly improves classification performance compared with no transform. 
However, with increasing noise, TRAIT features outperform LRT features, showing greater robustness to model mismatch.

\subsubsection{Face Images with \textbf{non-frontal} Poses}
It is known that human frontal face images are well modeled by subspaces. For example, the Yale-B face in section~\ref{sec:yaleB}, where LRT slightly outperforms TRAIT. 
Now we further compare the performance of TRAIT and LRT in more mismatched cases by introducing non-frontal face images. We validate performance on three publicly available datasets, PIE~\cite{PIEdata}, UMIST\footnote{http://www.sheffield.ac.uk/eee/research/iel/research/face} and ORL\footnote{http://www.cl.cam.ac.uk/research/dtg/attarchive/facedatabase.html}. 
All of them have a considerable number of non-frontal face images. 
Fig.~\ref{fig:PIE} shows one subject from each database with different poses.

The PIE dataset includes $18562$ $64\times 48$ images of 68 subjects. Each image is labeled with one of 13 different pose tags. 
We randomly select 7 pose tags and the images of these tags are used as training samples. The rest are used in testing. 
UMIST comprises $575$ $112\times 92$ images of $20$ subjects, and ORL comprises $400$ $112\times 92$ images of $40$ subjects.
These two datasets have no pose tags. 
We split the UMIST and ORL datasets using the strategy followed for the Yale-B dataset in Section~\ref{sec:yaleB}. We derive 1000-dimensional features for each of random projection, LDA, LRT and TRAIT. Table~\ref{tab:moreFaces} lists accuracies of NSC classification for the different algorithms.
\begin{table}[h!]
	\centering 
	\caption{NSC accuracy on original and 1000 dimensional (compressed) extracted features}
	\begin{tabular}{c|c|c|c}
		\hline\hline
				  & PIE     & UMIST   & ORL \\
				  \hline
		Original  & 74.57\% & 96.14\% & 95.50\% \\
		random    & 72.14\% & 95.44\% & 94.50\% \\
		LDA       & 40.10\% & 84.91\% & 92.00\% \\
		LRT       & 70.80\% & 96.84\% & 95.00\% \\
		TRAIT     & \bf 76.11\% & \bf 97.90\% & \bf 97.00\% \\
		\hline\hline
	\end{tabular}
	\label{tab:moreFaces}
\end{table}

In all cases, TRAIT has the highest classification accuracy and outperforms LRT. 
LRT optimizes the rank (its convex relaxation), which is critical for reducing classification error in the high SNR regime. 
However, in this low SNR regime, TRAIT gains more discrimination via explicitly ``orthogonalizing" between the classes. 
The criteria employed by TRAIT do not depend on the specific SNR regime and therefore are more robust. 
}
%========================================================================================
\section{Conclusion}
\label{sec:conclu}
%========================================================================================
\textcolor{\chcolor}{In a low-rank Gaussian Mixture Model, we have explored how the probability of misclassification is governed by principal angles between subspaces. In the low-noise regime, the Bhattacharyya upper bound on misclassification is determined by the product of the sines of the principal angles.} In the high/moderate-noise regime it is determined by the sum of the squares of the sines of the principal angles. Analysis of the Nearest Subspace Classifier connected reliability of classification to the distribution of signal energy across principal vectors. Classification was shown to be more reliable when more signal energy is associated with principal vectors corresponding to large principal angles. This observation motivated the design of a transform, TRAIT, that achieves superior classification performance by enlarging principal angles and preserving intra-class structure. \textcolor{\chcolor}{Finally we showed that TRAIT complements a prior approach that enlarge inter-class distance while suppressing intra-class dispersion, and that it is more robust to model mismatch.}

%Under low-rank GMM assumption, we show how the principal angles relates to miss-classification probability.
%Specifically, in low noise regime, it is the product of the sine principal angles that matters. 
%In contrast, in high noise regime, sum of the squared sine principal angles plays a role.
%Further, analysis of the NSC classifier, which relaxes GMM assumption, indicates that allocating bigger signal energy to more discriminative modes reduces classification error.
%Finally, motived by the theoretical results, we designed a learned transform, \trait, that enlarges the subspace angles and achieves state-of-the-art classification performance.

%========================================================================================
\appendices

%========================================================================================
\section{Proof of high SNR case}
\label{proof:highSNR}
%========================================================================================
\noindent{\bf Proof of Theorem~\ref{thm:highSNR}} We have
\[
\begin{array}{ll}
\det\boldsymbol\Sigma_1=&(\sigma^2)^{n-d}\prod_{i=1}^d \left(\lambda_{1,i}+\sigma^2\right),\\
\det\boldsymbol\Sigma_2=&(\sigma^2)^{n-d}\prod_{i=1}^d \left(\lambda_{2,i}+\sigma^2\right).
\end{array}
\]
%and
%\[
%\begin{aligned}
%
%\end{aligned}\]
%%
Let the SVD of
$
\U_{1,\cap}\boldsymbol\Lambda_{1,\cap} \U_{1,\cap}^\top+\U_{1,\backslash}\boldsymbol\Lambda_{1,\backslash} \U_{1,\backslash}^\top+\U_{2,\cap}\boldsymbol\Lambda_{2,\cap} \U_{2,\cap}^\top+\U_{2,\backslash}\boldsymbol\Lambda_{2,\backslash} \U_{2,\backslash}^\top
$
be ${\bf Z}\boldsymbol\Lambda {\bf Z}^\top$,
where $\boldsymbol\Lambda=\diag\{\lambda_1,\dots,\lambda_{2d-r}\}$.
Then,
\[
\det\left(\frac{\boldsymbol\Sigma_1+\boldsymbol\Sigma_2}{2}\right)
=(\sigma^2)^{n-2d+r}\prod_{i=1}^{2d-r}\left({\lambda_i\over2}+\sigma^2\right).
\]
Substituting the above into the Bhattacharyya bound, we have
\ben
\begin{aligned}
	P_e \leq & {1\over2}(\sigma^2)^{d-r\over2}\cdot \left[\frac{\sqrt{\prod_{i=1}^d \left(\lambda_{1,i}+\sigma^2\right)\prod_{i=1}^{d} \left(\lambda_{2,i}+\sigma^2\right)}}{\prod_{i=1}^{2d-r}\left({\lambda_i\over2}+\sigma^2\right)}\right]^{1\over2}\\
	=& (\sigma^2)^{d-r\over2}\cdot 2^{{2d-r\over2}-1}\left[\frac{\sqrt{\prod_{i=1}^d \lambda_{1,i}\prod_{i=1}^d \lambda_{2,i}}}{\prod_{i=1}^{2d-r}\lambda_i}\right]^{1\over2}\\
	& + o\left((\sigma^2)^{d-r\over2}\right).
	\label{eq:Bbound_partial_expand}
\end{aligned}
\een
Our objective is to expand $\prod_{i=1}^{2d-r}\lambda_i$ in terms of principal angles.
Since the image of $\U_{1,\cap}$ (or $\U_{2,\cap}$) is orthogonal to $\U_{1,\backslash}$ and $\U_{2,\backslash}$,
\[
\begin{aligned}
	\prod_{i=1}^{2d-r}\lambda_i =& \pdet(\U_{1,\cap}\boldsymbol\Lambda_{1,\cap}\U_{1,\cap}^\top + \U_{2,\cap}\boldsymbol\Lambda_{2,\cap}\U_{2,\cap}^\top)\\
	&\cdot\pdet([\U_{1,\backslash}\boldsymbol\Lambda_{1,\backslash}^{1\over 2}~\U_{2,\backslash}\boldsymbol\Lambda_{2,\backslash}^{1\over 2}][\U_{1,\backslash}\boldsymbol\Lambda_{1,\backslash}^{1\over 2}~\U_{2,\backslash}\boldsymbol\Lambda_{2,\backslash}^{1\over 2}]^\top)\\
	=&\pdet\left(\U_{1,\cap}\boldsymbol\Lambda_{1,\cap}\U_{1,\cap}^\top + \U_{2,\cap}\boldsymbol\Lambda_{2,\cap}\U_{2,\cap}^\top\right)\\
	&\cdot\det([\U_{1,\backslash}\boldsymbol\Lambda_{1,\backslash}^{1\over 2}~\U_{2,\backslash}\boldsymbol\Lambda_{2,\backslash}^{1\over 2}]^\top[\U_{1,\backslash}\boldsymbol\Lambda_{1,\backslash}^{1\over 2}~\U_{2,\backslash}\boldsymbol\Lambda_{2,\backslash}^{1\over 2}]),
\end{aligned}
\label{eq:expandNominator1}
\]
where we assume $n\geq 2(d-r)$ in order to derive the second equality, which simplifies as follows:
\ben
\begin{aligned}
	&\det([\U_{1,\backslash}\boldsymbol\Lambda_{1,\backslash}^{1\over 2}~\U_{2,\backslash}\boldsymbol\Lambda_{2,\backslash}^{1\over 2}]^\top[\U_{1,\backslash}\boldsymbol\Lambda_{1,\backslash}^{1\over 2}~\U_{2,\backslash}\boldsymbol\Lambda_{2,\backslash}^{1\over 2}])\\
	& = \det\left(\bmx\boldsymbol\Lambda_{1,\backslash} & \boldsymbol\Lambda_{1,\backslash}^{1\over2}\U_{1,\backslash}^\top \U_{2,\backslash}\boldsymbol\Lambda_{2,\backslash}^{1\over2}\\ \boldsymbol\Lambda_{2,\backslash}^{1\over2}\U_{2,\backslash}^\top \U_{1,\backslash}\boldsymbol\Lambda_{1,\backslash}^{1\over2} & \boldsymbol\Lambda_{2,\backslash} \emx\right)\\
	&=\det(\boldsymbol\Lambda_{1,\backslash})\det\left(\boldsymbol\Lambda_{2,\backslash}- \right.\\
	& \qquad\qquad~ \left.\boldsymbol\Lambda_{2,\backslash}^{1\over2}\U_{2,\backslash}^\top \U_{1,\backslash}\boldsymbol\Lambda_{1,\backslash}^{1\over2}
	\boldsymbol\Lambda_{1,\backslash}^{-1}\boldsymbol\Lambda_{1,\backslash}^{1\over2}\U_{1,\backslash}^\top \U_{2,\backslash}\boldsymbol\Lambda_{2,\backslash}^{1\over2}\right)\\
	&=\det(\boldsymbol\Lambda_{1,\backslash})\det\left(\boldsymbol\Lambda_{2,\backslash}^{1\over2}(\I-\U_{2,\backslash}^\top \U_{1,\backslash}\U_{1,\backslash}^\top \U_{2,\backslash})\boldsymbol\Lambda_{2,\backslash}^{1\over2} \right)\\
	%&=\det(\Lambda_{1,n}) \det(\Lambda_{2,n})\det(\I-U_{2,n}^\top U_{1,n}U_{1,n}^\top U_{2,n})\\
	&=\prod_{i=1}^{d-r} \boldsymbol\lambda_{1,\backslash,i} \cdot \prod_{i=1}^{d-r} \boldsymbol\lambda_{2,\backslash,i} \cdot \prod_{i=r+1}^d \sin^2\theta_i.
\end{aligned}
\label{eq:expandNominator2}
\een
The last equality follows from the observation that the eigenvalues of $\U_{2,\backslash}^\top \U_{1,\backslash}\U_{1,\backslash}^\top \U_{2,\backslash}$ are $\cos^2\theta_{r+1},\dots,\cos^2\theta_{d}$.
The theorem now follows by substituting Eq. \eqref{eq:expandNominator2} into Eq. \eqref{eq:Bbound_partial_expand}. 

\hfill$\qed$

%========================================================================================
\section{Proof of Low SNR case}
\label{proof:lowSNR}
%========================================================================================
We first state and prove (for completeness) two preliminary lemmas that are needed to characterize classification error.
\begin{lemma}
	\label{TaylorExp}
	Let ${\bf D}\in\R^{n\times n}$ be any positive semi-definite matrix with all eigenvalues smaller than $1$, then 
	\[
	\tr({\bf D})-{1\over2}\tr({\bf D}^2) \leq \ln\det(\I_n+{\bf D}) \leq \tr({\bf D})-{1\over4}\tr({\bf D}^2).
	\]	
\end{lemma}
\begin{proof}
	Denote the nonnegative eigenvalues of $\mathbf D\succeq 0$ as $d_1,\dots,d_n$, where $d_1,\dots,d_n\leq1$.
	Then 
	\[\ln\det(\I_n+{\bf D})=\ln\prod_{i=1}^n(1+d_i)=\sum_{i=1}^n\ln(1+d_i).\]
	Since
	$
	x-\frac{x^2}{2}\leq\ln(1+x)\leq x-\frac{x^2}{4}
	$
	for all $x\in[0,1]$, we obtain
	\[
	\sum_i d_i-\frac{d_i^2}{2}
	\leq\ln\det(\I_n+{\bf D})
	\leq\sum_i d_i-\frac{d_i^2}{4},
	\]
	which reduces to
	\[\tr({\bf D})-{1\over 2}\tr({\bf D}^2) \leq \ln\det(\I_n+{\bf D}) \leq \tr({\bf D})-{1\over 4}\tr({\bf D}^2). \]
	This bound is very tight when all the $d_i$'s approach $0$.
\end{proof}
\begin{lemma}
	\label{traceIneq}
	Suppose $\U\in\R^{n\times d}, \V\in\R^{n\times d}$ are two orthonormal bases and that $\boldsymbol\Phi\in\R^{d\times d}$, $\boldsymbol\Psi\in\R^{d\times d}$ are diagonal with nonnegative decreasing diagonal elements $\phi_1,\dots,\phi_d$ and $\psi_1,\dots,\psi_d$ respectively. Denote the $i$-th principal angle between $\U$ and $\V$ as $\theta_i$ where $i=1,\dots,d$. Then 
	\[\phi_d\psi_d \sum_i\cos^2\theta_i \leq \tr(\U\boldsymbol\Phi \U^\top \V\boldsymbol\Psi V^\top)\leq\phi_1\psi_1\sum_i\cos^2\theta_i.\]
\end{lemma}
\begin{proof}
	Let the Singular Value Decomposition of $\U^\top \V$ be ${\bf JCH}^\top$, then $\tr(\U^\top \V\V^\top \U)=\tr({\bf C}^2)=\sum_i\cos^2\theta_i$. 
	We have 
	\[
	\begin{aligned}
	&\tr(\U\boldsymbol\Phi \U^\top \V\boldsymbol\Psi \V^\top)=\tr(\boldsymbol\Phi \U^\top \V\boldsymbol\Psi \V^\top \U)\\
	&=\tr(\boldsymbol\Phi {\bf JCH}^\top \boldsymbol\Psi {\bf HCJ}^\top)=\tr({\bf J}^\top\boldsymbol\Phi {\bf J C H}^\top\boldsymbol\Psi \bf{ H C}).
	\end{aligned}
	\]
	For any two positive semidefinite matrices $\A, {\bf B}\in\R^{m\times m}$, let the maximum and minimum eigenvalues of $\A$ be $\lambda_1(\A), \lambda_m(\A)$ respectively,
	then by \cite{Kleinman1968}
	\[
	\lambda_m(\A)\tr({\bf B})\leq\tr({\bf AB})\leq\lambda_1(\A)\tr({\bf B}).
	\]
	Hence,
	\[
	\begin{aligned}
	\tr(\U\boldsymbol\Phi \U^\top \V\boldsymbol\Psi \V^\top)&\leq\phi_1\tr({\bf C H}^\top\boldsymbol\Psi {\bf H C})=\phi_1\tr({\bf H}^\top\boldsymbol\Psi {\bf H C}^2)\\
	&\leq\phi_1\psi_1\tr({\bf C}^2)=\phi_1\psi_1\sum_i\cos^2\theta_i.
	\end{aligned}  
	\]
	The lower bound can be proved in the same way. 
	This bound becomes tight when the diagonal elements of $\boldsymbol\Phi$ and $\boldsymbol\Psi$ are uniform.
\end{proof}

{\noindent \bf Proof of Theorem~\ref{thm:lowSNR}}
We are now ready to prove theorem~\ref{thm:lowSNR}.	
We expand $K$ in Eq.~\eqref{eq:Bbound} as
\ben
K={1\over2}\ln\det\left(\boldsymbol\Sigma_1+\boldsymbol\Sigma_2\over2  \right)-{1\over4}(\ln\det\boldsymbol\Sigma_1+\ln\det\boldsymbol\Sigma_2).
\een
The second term becomes:
	\ben
	-{1\over4}\left[\sum_{i=1}^d\ln\left(1+{\lambda_{1,i}\over\sigma^2}\right)
		+\sum_{i=1}^d\ln\left(1+{\lambda_{2,i}\over\sigma^2}\right) \right]-{n\over2}\ln(\sigma^2),
	\label{eq:denominator}
	\een
	and we use Lemma~\ref{TaylorExp} to bound the first term. Note that
	\ben
	\begin{aligned}
		&{1\over2}\ln\det\left(\boldsymbol\Sigma_1+\boldsymbol\Sigma_2\over2\right)\\
		&=
		{1\over2}\ln\det\left[\sigma^2\left(\I+\frac{\U_1\Lambda_1\U_1^\top+\U_2\boldsymbol\Lambda_2\U_2^\top}{2\sigma^2}  \right)  \right]\\
		&={n\over2}\ln(\sigma^2)+{1\over2}\ln\det\left( \I+\frac{\U_1\boldsymbol\Lambda_1\U_1^\top+\U_2\boldsymbol\Lambda_2\U_2^\top}{2\sigma^2} \right).
	\end{aligned}
	\label{eq:numeratorLower}
	\een
	Let ${\bf D}\triangleq \frac{\U_1\boldsymbol\Lambda_1\U_1^\top+\U_2\boldsymbol\Lambda_2\U_2^\top}{2\sigma^2}$.
	We apply Lemma~\ref{TaylorExp} to bound ${1\over2}\ln\det\left(\boldsymbol\Sigma_1+\boldsymbol\Sigma_2\over2\right)$:
	\ben
	\begin{aligned}
		&{n\over2}\ln(\sigma^2)+{1\over2}\left[\tr({\bf D})-{1\over2}\tr({\bf D}^2)\right] \leq {1\over2}\ln\det\left(\boldsymbol\Sigma_1+\boldsymbol\Sigma_2\over2\right)\\
		&\leq {n\over 2}\ln(\sigma^2)+{1\over2}\left[\tr({\bf D})-{1\over4}\tr({\bf D}^2)  \right],
	\end{aligned}
	\een
	Expanding $\tr({\bf D})$ gives
	\ben\begin{aligned} 
		&{n\over2}\ln(\sigma^2)+{1\over4}\left[\sum_{i=1}^d{\lambda_{1,i}\over\sigma^2} +\sum_{i=1}^d{\lambda_{2,i}\over\sigma^2}\right]-{1\over4}\tr({\bf D}^2)\\
		&\leq {1\over2}\ln\det\left(\boldsymbol\Sigma_1+\boldsymbol\Sigma_2\over2\right)\\
		&\leq
		{n\over2}\ln(\sigma^2)+{1\over4}\left[\sum_{i=1}^d{\lambda_{1,i}\over\sigma^2} +\sum_{i=1}^d{\lambda_{2,i}\over\sigma^2}\right]-{1\over8}\tr({\bf D}^2).
	\end{aligned}
	\een
	Note that
	\ben
	\tr({\bf D}^2)={1\over4\sigma^4} \bigg(\sum_{i=1}^d \lambda_{1,i}^2 +\sum_{i=1}^d \lambda_{2,i}^2  +2\tr(\U_1\boldsymbol\Lambda_1\U_1^\top \U_2\boldsymbol\Lambda_2\U_2^\top) \bigg).
	\label{eq:trD2}
	\een
	Envoking Lemma~\ref{traceIneq} to bound the last term of the above:
	\ben\begin{aligned}
		\tr(\U_1\boldsymbol\Lambda_1\U_1^\top \U_2\boldsymbol\Lambda_2\U_2^\top) &\geq \lambda_{1,d}\lambda_{2,d}\sum_i\cos^2\theta_i\\	
		\tr(\U_1\boldsymbol\Lambda_1\U_1^\top \U_2\boldsymbol\Lambda_2\U_2^\top) &\leq \lambda_{1,1}\lambda_{2,1}\sum_i\cos^2\theta_i
		\label{eq:useLemma}
	\end{aligned}
	\een
	Combining Eq.~\eqref{eq:denominator} to \eqref{eq:useLemma}, we obtain upper and lower bounds on $K$,\\
	\subsubsection{Upper bound}
	\ben\begin{aligned}
		K\leq&{1\over4}\left[\sum_{i=1}^d{\lambda_{1,i}\over\sigma^2} +\sum_{i=1}^d{\lambda_{2,i}\over\sigma^2}\right]\\
		&-{1\over32\sigma^4}\bigg(\sum_{i=1}^d \lambda_{1,i}^2 +\sum_{i=1}^d \lambda_{2,i}^2	+ 2\lambda_{1,d}\lambda_{2,d}\sum_{i=1}^d\cos^2\theta_i \bigg)\\
		&-{1\over4}\left[\sum_{i=1}^d\ln\left(1+{\lambda_{1,i}\over\sigma^2}\right)
		+\sum_{i=1}^d\ln\left(1+{\lambda_{2,i}\over\sigma^2}\right) \right] \\
		=& {1\over4}\left[\sum_{i=1}^d{\lambda_{1,i}\over\sigma^2}-{1\over2}\sum_{i=1}^d \left({\lambda_{1,i}\over 2\sigma^2}\right)^2-\sum_{i=1}^d\ln\left(1+{\lambda_{1,i}\over\sigma^2}\right) \right]\\
		&+{1\over4}\left[\sum_{i=1}^d{\lambda_{2,i}\over\sigma^2}-{1\over2}\sum_{i=1}^d \left({\lambda_{2,i}\over 2\sigma^2}\right)^2
		-\sum_{i=1}^d\ln\left(1+{\lambda_{2,i}\over\sigma^2}\right) \right]\\
		& - {1\over 16\sigma^4}\lambda_{1,d}\lambda_{2,d}\sum_{i=1}^d\cos^2\theta_i\\
		\triangleq& {1\over\sigma^4}\left(c_2-{1\over16}\lambda_{1,d}\lambda_{2,d}\sum_{i=1}^d\cos^2\theta_i\right).
	\end{aligned}
	\een
	
	\subsubsection{Lower bound}
	\ben
	\begin{aligned}
		K\geq&
		{1\over4}\left[\sum_{i=1}^d{\lambda_{1,i}\over\sigma^2} +\sum_{i=1}^d{\lambda_{2,i}\over\sigma^2}\right]\\
		&-{1\over16\sigma^4}\bigg(\sum_{i=1}^d \lambda_{1,i}^2 +\sum_{i=1}^d \lambda_{2,i}^2	+ 2\lambda_{1,1}\lambda_{2,1}\sum_{i=1}^d\cos^2\theta_i \bigg)\\
		&-{1\over4}\left[\sum_{i=1}^d\ln\left(1+{\lambda_{1,i}\over\sigma^2}\right)
		+\sum_{i=1}^d\ln\left(1+{\lambda_{2,i}\over\sigma^2}\right) \right] \\
		=& {1\over4}\left[\sum_{i=1}^d{\lambda_{1,i}\over\sigma^2}-\sum_{i=1}^d \left({\lambda_{1,i}\over 2\sigma^2}\right)^2
		-\sum_{i=1}^d\ln\left(1+{\lambda_{1,i}\over\sigma^2}\right) \right]\\
		&+{1\over4}\left[\sum_{i=1}^d{\lambda_{2,i}\over\sigma^2}-\sum_{i=1}^d \left({\lambda_{2,i}\over 2\sigma^2}\right)^2
		-\sum_{i=1}^d\ln\left(1+{\lambda_{2,i}\over\sigma^2}\right) \right] \\
		&- {1\over 8\sigma^4}\lambda_{1,1}\lambda_{2,1}\sum_{i=1}^d\cos^2\theta_i\\
		\triangleq& {1\over\sigma^4}\left(c_3-{1\over 8}\lambda_{1,d}\lambda_{2,d}\sum_{i=1}^d\cos^2\theta_i  \right).
		%\triangleq& C_0(\sigma^2)-{1\over 8\sigma^4}\lambda_{1,1}\lambda_{2,1}\sum_{i=1}^d\cos^2\theta_i
	\end{aligned}
	\label{eq:lower1}
	\een
	Negating $K$ and exponentiating gives theorem~\ref{thm:lowSNR}.
	\hfill$\qed$	
	
%========================================================================================
\section{\textcolor{\chcolor}{Proof of Moderate SNR Case}}
\label{sec:proofModerateSNR}
%========================================================================================
{\color{\chcolor}
{\noindent \bf Proof of Lemma~\ref{lemma:ell(p)}} consider the function 
\[
f(\lambda_i)=\ln(1+\lambda_i)-\ln(1+p)-{1\over 1+p}(\lambda_i-p)+{1\over(1+p)^2}(\lambda_i-p)^2,
\]
defined in $[0,p]$. Its derivative is 
\[
f^\prime(\lambda_i)={1\over 1+\lambda_i}-{1\over 1+p}+{2(\lambda_i-p)\over(1+p)^2}=\frac{(p-\lambda_i)(p-1-2\lambda_i)}{(1+\lambda_i)(1+p)^2},
\]
which is positive in $\left[0,\frac{p-1}{2}\right)$ and negative in $\left(\frac{p-1}{2},p\right]$.
Therefore, $f(\lambda_i)$ is monotonically increasing in $\left[0,\frac{p-1}{2}\right)$ and decreasing in $\left(\frac{p-1}{2},p\right]$.
Further, $f(p)=0$ and $f(0)=-\ln(1+p)+{p\over 1+p}+{p^2\over (1+p)^2}$ whose sign depends on the value of $p$.
The shape of $f(\lambda_i)$ is now characterized. There exists $L<{p-1\over2}$ such that $f(\lambda_i)\geq 0$ when $\lambda_i\in[L,p]$. 

\hfill$\qed$

Before proving theorem~\ref{thm:moderateSNR}, we need to bound $\lambda_i$ using Weyl's inequality~\cite{MtxAnal}.
\begin{lemma}[Weyl's inequality~\cite{MtxAnal}]
	Let $\mathbf M$ and $\mathbf P$ be two $n\times n$ Hermitian matrices, with eigenvalues 
	$\mu_1\geq\dots\geq\mu_n$ and $\nu_1\geq\dots\geq\nu_n$ respectively. Denote the eigenvalues of $\mathbf M+\mathbf P$ by
	$\gamma_1\geq\dots\geq\gamma_n$. Then
	\[\max(\mu_i+\nu_n,\nu_i+\mu_n)\leq \gamma_i\leq \min(\mu_i+\nu_1,\nu_i+\mu_1).\]
\end{lemma}
\noindent{\bf Proof of Theorem~\ref{thm:moderateSNR}} 
Since ${p\over c(p)}\leq {\lambda_{1,i}\over\sigma^2}, {\lambda_{2,i}\over\sigma^2} \leq p$,
by the Weyl's inequality, ${p\over2c(p)}={p/c(p)+0\over2}\leq \lambda_i\leq {p+p\over2}=p$.
Further, since $1\leq c(p)\leq {p\over2L(p)}$, we have $\lambda_1,\dots,\lambda_{2d-r}\in [L(p),p]$. 
By definition of $L(p)$, we can invoke Eq.~\eqref{eq:expansion_p} in Lemma~\ref{lemma:ell(p)} to obtain
\ben
\begin{aligned}
	&\ln\det\left({{\boldsymbol\Sigma}_1+{\boldsymbol\Sigma}_2\over2}\right) =\sum_{i=1}^{2d-r}\ln(1+\lambda_i)+n\ln(\sigma^2)\\
	&\geq (2d-r)\ln(1+p)+{\tr\mathbf D-p(2d-r)\over1+p}\\
	&-\frac{\tr \mathbf D^2-2p\tr\mathbf D+p^2(2d-r)}{(1+p)^2}+n\ln(\sigma^2).
\end{aligned}
\label{eq:numerator}
\een
Notice $\tr \mathbf D={1\over2}\sum_i \left( \frac{\lambda_{1,i}}{\sigma^2} + \frac{\lambda_{2,i}}{\sigma^2} \right)$, and by Eq.~\eqref{eq:trD2} and \eqref{eq:useLemma},
$\tr \mathbf D^2\leq {1\over4\sigma^4}\left(\sum_i\lambda_{1,i}^2+\lambda_{2,i}^2+2\lambda_{1,1}\lambda_{2,1}\sum_i\cos^2\theta_i\right)$.
Substituting these into Eq.~\eqref{eq:numerator}, we get
\[\begin{aligned}
	&\ln\det\left(\frac{\boldsymbol{\Sigma}_1+\boldsymbol{\Sigma}_2}{2}\right)\\
	\geq& n\ln(\sigma^2)+ (2d-r)\left[\ln(1+p)-\frac{p}{1+p}-\frac{p^2}{(1+p)^2}\right]\\
	&+{1+3p\over 2\sigma^2(1+p)^2}\left( \sum_i \lambda_{1,i} + \lambda_{2,i}\right)\\
	&-{1\over4\sigma^4(1+p)^2}\left(\sum_i\lambda_{1,i}^2+\lambda_{2,i}^2+2\lambda_{1,1}\lambda_{2,1}\sum_i\cos^2\theta_i\right)
\end{aligned}\]
Substituting the above into the Bhattacharyya bound~\eqref{eq:Bbound} yields an upper bound on $P_e$, of the form given in Theorem~\ref{thm:moderateSNR}. In particular,
\[
	c_4 = {1\over2}\left[\ln(1+p)-\frac{p}{1+p}-\frac{p^2}{(1+p)^2}\right],
\]
and
\[\begin{aligned} 
	c_5 = & -{1+3p\over4\sigma^2(1+p)^2}\sum_i\left(\lambda_{1,i}+\lambda_{2,i}\right)+{\sum_i\lambda_{1,i}^2+\lambda_{2,i}^2\over8\sigma^4(1+p)^2}\\
	&+{1\over4}\sum_i\left[\ln\left(1+{\lambda_{1,i}\over\sigma^2}\right) + \ln\left(1+{\lambda_{2,i}\over\sigma^2}\right)  \right].
\end{aligned}\] 
\hfill$\qed$
}
%========================================================================================
\section{Analysis of NSC}
\label{sec:proof_NSC}
%========================================================================================
\noindent{\bf Proof of Lemma~\ref{lemma:independence}}
Since that the joint distribution of $[a_i~ a_j]^\top$, $[b_i~ b_j]^\top$, $[a_i~ b_j]^\top$ and $[a_i+b_i~a_i-b_i]^\top$
are all Gaussian, it suffices to show that all covariance are diagonal. 
For any $i\neq j$,
\[
\begin{aligned}
&\bmx a_i \\ a_j \emx \sim \N\left(\bmx \alpha_i \\ \alpha_j \emx, \sigma^2\I_2  \right)\quad
\bmx b_i \\ b_j \emx \sim \N\left(\bmx \cos\theta_i\alpha_i \\ \cos\theta_j\alpha_j \emx, \sigma^2\I_2  \right) \\
&\bmx a_i \\ b_j \emx \sim \N\left(\bmx \alpha_i \\ \cos\theta_j\alpha_j \emx, \sigma^2\I_2  \right). \\
\end{aligned}
\]
For any $i$,
\ben
\begin{aligned}
&\bmx a_i \\ b_i \emx \sim \N\left(\bmx \alpha_i \\ \cos\theta_i\alpha_i \emx, \sigma^2\bmx 1 & \cos\theta_i \\\cos\theta_i & 1  \emx  \right)\\
&\bmx a_i+b_i \\ a_i-b_i\emx \sim \N\left(\bmx (1+\cos\theta_i)\alpha_i \\ (1-\cos\theta_i)\alpha_i \emx, \right.\\
&\qquad\qquad\qquad\qquad 2\sigma^2\left.\bmx (1+\cos\theta_i) & 0 \\ 0 & (1-\cos\theta_i) \emx     \right),
\end{aligned}
\label{eq:independence}
\een
which concludes the proof.
\hfill $\qed$

\noindent{\bf Proof of Lemma~\ref{lemma:ourProductNormal}}
As $\sigma^2\to0$, the mean-covariance ratios of both $a_i+b_i$ and $a_i-b_i$ tend to infinity.
Therefore, applying Lemma~\ref{lemma:productNormal} to Eq.~\eqref{eq:independence} (see proof of Lemma~\ref{lemma:independence}), we have $(a_i+b_i)(a_i-b_i)\sim\N\left(\sin^2\theta_i \alpha_i^2, 4\sigma^2\sin^2\theta_i(\alpha_i^2+\sigma^2) \right)$. 
Applying the independence between $(a_i+b_i)(a_i-b_i)$ and $(a_j+b_j)(a_j-b_j)$ ($i\neq j$), 
we obtain the desired result by summing the mean and variance over all $i$.
\hfill $\qed$

\noindent{\bf Proof of Theorem~\ref{thm:NSCbound}}
We prove the theorem by deriving upper bounds on $\Pr(\C_2|\C_1,\boldsymbol\alpha)$ and $\Pr(\C_1|\C_2,\boldsymbol\alpha)$.
\ben\begin{aligned}
	&{\Pr}(\C_2|\C_1,\boldsymbol\alpha)={\Pr}\left(\sum_i(a_i+b_i)(a_i-b_i)\leq 0\right)\\
	&={\Pr}\left(\frac{\sum_i(a_i+b_i)(a_i-b_i)-\sum_i \sin^2\theta_i\alpha_i^2}{2\sigma\sqrt{\sum_i\sin^2\theta_i (\alpha_i^2+\sigma^2)}} \leq  \right.\\
	&\qquad\qquad \left.-\frac{\sum_i \sin^2\theta_i\alpha_i^2}{2\sigma\sqrt{\sum_i\sin^2\theta_i (\alpha_i^2+\sigma^2)}}\right).\\
\end{aligned} 
\label{eq:relation1}
\een
As $\sigma\to0$, the term to the left of ``$\leq$" in the last line of Eq.~\eqref{eq:relation1} is standard normal distributed. Therefore we can invoke the Gaussian tail bound to obtain
\ben\begin{aligned}
	&{\Pr}(\C_2|\C_1,\boldsymbol\alpha)\\
	&= {\Pr}\left(\frac{\sum_i(a_i+b_i)(a_i-b_i)-\sum_i \sin^2\theta_i\alpha_i^2}{2\sigma\sqrt{\sum_i\sin^2\theta_i (\alpha_i^2+\sigma^2)}} \geq  \right.\\
	&\qquad\qquad \left.\frac{\sum_i \sin^2\theta_i\alpha_i^2}{2\sigma\sqrt{\sum_i\sin^2\theta_i (\alpha_i^2+\sigma^2)}}\right)\\
	&\leq{1\over2}\exp\left[-\frac{\left(\sum_i \sin^2\theta_i\alpha_i^2 \right)^2}{8\sigma^2\sum_i\sin^2\theta_i (\alpha_i^2+\sigma^2 )}   \right].
\end{aligned}\label{eq:conditionONalpha1}\een
$\Pr(\C_1|\C_2,\boldsymbol\alpha)$ can be upper bounded in the same manner:
\ben
	{\Pr}(\C_1|\C_2,\boldsymbol\alpha)\leq {1\over2}\exp\left[-\frac{\left(\sum_i \sin^2\theta_i\alpha_i^2 \right)^2}{8\sigma^2\sum_i\sin^2\theta_i (\alpha_i^2+\sigma^2 )} \right].
	\label{eq:conditionONalpha2}
\een
Therefore,
\ben\begin{aligned} 
	P_e =& {1\over2} \int{\Pr}(\C_2|\C_1,\boldsymbol\alpha)p(\boldsymbol\alpha)d\boldsymbol\alpha + {1\over2} \int{\Pr}(\C_1|\C_2,\boldsymbol\alpha)q(\boldsymbol\alpha)d\boldsymbol\alpha\\
	\leq&\int{1\over2}\exp\left[-\frac{\left(\sum_i \sin^2\theta_i\alpha_i^2 \right)^2}{8\sigma^2\sum_i\sin^2\theta_i (\alpha_i^2+\sigma^2 )} \right] \frac{p(\boldsymbol\alpha)+q(\boldsymbol\alpha)}{2} d\boldsymbol\alpha \\
	\triangleq& \int\Er(\theta,\boldsymbol\alpha,\sigma) \frac{p(\boldsymbol\alpha)+q(\boldsymbol\alpha)}{2} d\boldsymbol\alpha,
\end{aligned}
\label{eq:Pe}
\een
which concludes the proof.
\hfill$\qed$

%========================================================================================
\section{}
\label{proof:Prop1}
%========================================================================================
{\noindent \bf Proof of Proposition 1} 
	Observe that 
	\[\|\X^\top {\bf P} \X-{\bf T}\|_F^2=\|(\X^\top\otimes \X^\top)\vect({\bf P})-\vect({\bf T})\|_2^2,\]
	is a least squares problem with minimizer
	\[\vect({\bf P}^\star)=(\X^\top \otimes \X^\top)^\dagger \vect({\bf T})=\X^\top,\]
	which can be rearranged to give
	\[{\bf P}^\star=(\X^\top)^\dagger {\bf T} [(\X^\top)^\dagger]^\top=(\X\X^\top)^{-1} \X{\bf T}\X^\top (\X\X^\top)^{-1} \succeq 0.\] 
\hfill$\qed$

\ifCLASSOPTIONcaptionsoff
  \newpage
\fi

\bibliographystyle{IEEEtran}
\bibliography{refs}

%\begin{IEEEbiography}{Jiaji Huang}
%Biography text here.
%\end{IEEEbiography}
%
%
%\begin{IEEEbiography}{Robert Calderbank}
%Biography text here.
%\end{IEEEbiography}

\end{document}